\documentclass{article}

\usepackage{xcolor}
\newcommand{\n}[1]{\textcolor{red}{#1}}

\usepackage{amsthm}
\usepackage{amsmath}
\usepackage{amssymb}
\PassOptionsToPackage{numbers, compress}{natbib}

\usepackage{relsize}
\usepackage{tikz}
\usepackage[percent]{overpic}
\usepackage{subcaption}
\usepackage{bm}

\newtheorem{theorem}{Theorem}
\newtheorem{proposition}{Proposition}

\renewcommand{\vec}[1]{\ensuremath{\bm{#1}}}
\newcommand{\mat}[1]{\ensuremath{\mathbf{#1}}}

\newcommand{\transpose}{\ensuremath{^\top}}
\newcommand{\bigsum}[3]{\mathlarger{\sum}_{#1 = #2}^{#3}}
\newcommand{\bigsumm}[1]{\mathlarger{\sum}_{#1}}

\usepackage{microtype}
\usepackage{graphicx}
\usepackage{booktabs} 

\usepackage[accepted]{icml2025}

\usepackage{hyperref}
\usepackage{amsmath}
\usepackage{amssymb}
\usepackage{mathtools}
\usepackage{amsthm}
\DeclareMathOperator{\Aut}{Aut}
\usepackage[capitalize,noabbrev]{cleveref}

\theoremstyle{plain}

\theoremstyle{definition}

\theoremstyle{remark}

\newtheorem{example}{Example}

\usepackage[textsize=tiny]{todonotes}

\newcommand{\vcentered}[1]{\begingroup
\setbox0=\hbox{#1}%
\parbox{\wd0}{\box0}\endgroup}

\icmltitlerunning{Boosting GNN Expressivity with Learnable Lanczos Constraints}

\begin{document}
\twocolumn[
\icmltitle{Boosting Graph Neural Network Expressivity with Learnable Lanczos Constraints}

\icmlsetsymbol{equal}{*}
\begin{icmlauthorlist}
    \icmlauthor{Niloofar Azizi}{vec}
    \icmlauthor{Nils Kriege}{vienna}
    \icmlauthor{Horst Bischof}{tug}
\end{icmlauthorlist}

\icmlaffiliation{vec}{Vector Institute, Canada}
\icmlaffiliation{vienna}{University of Vienna, Austria}
\icmlaffiliation{tug}{Graz University of Technology, Austria}

\icmlcorrespondingauthor{Niloofar Azizi}{niloofarazizi37@gmail.com}

\icmlkeywords{Machine Learning, ICML}

\vskip 0.3in

]

\printAffiliationsAndNotice{}

\begin{abstract}
Graph Neural Networks (GNNs) excel in handling graph-structured data but often underperform in link prediction tasks compared to classical methods, mainly due to the limitations of the commonly used message-passing principle. Notably, their ability to distinguish non-isomorphic graphs is limited by the 1-dimensional Weisfeiler-Lehman test. Our study presents a novel method to enhance the expressivity of GNNs by embedding induced subgraphs into the graph Laplacian matrix's eigenbasis.
We introduce a \textbf{L}earnable \textbf{L}anczos algorithm with \textbf{L}inear \textbf{C}onstraints
(LLwLC), proposing two novel subgraph extraction strategies: encoding vertex-deleted subgraphs and applying Neumann eigenvalue constraints. For the former, we demonstrate the ability to distinguish graphs that are indistinguishable by 2-WL, while maintaining efficient time complexity. The latter focuses on link representations enabling differentiation between $k$-regular graphs and node automorphism, a vital aspect for link prediction tasks. Our approach results in an extremely lightweight architecture, reducing the need for extensive training datasets. 
Empirically, our method improves performance in challenging link prediction tasks across benchmark datasets, establishing its practical utility and supporting our theoretical findings. Notably, LLwLC achieves $20\text{x}$ and $10\text{x}$ speedup by only requiring 5\% and 10\% data from the PubMed and OGBL-Vessel datasets while comparing to the state-of-the-art. 
\end{abstract}

\section{Introduction}
\label{introduction}

Graphs play a crucial role in various domains, representing linked data such as social networks~\citep{adamic2003friends}, citation networks~\citep{shibata2012link}, knowledge graphs~\citep{nickel2015review}, metabolic network reconstruction~\citep{oyetunde2017boostgapfill}, and user-item graphs in recommender systems~\citep{monti2017geometric}. Graph Neural Networks (GNNs) have emerged as state-of-the-art tools for processing graph-structured data. Message Passing Neural Networks (MPNNs) is the most prevalent technique within GNNs, which, relying on neighborhood aggregation, exhibit expressiveness no greater than the first-order Weisfeiler-Leman (1-WL) test~\citep{weisfeiler1968reduction,morris2019weisfeiler,Xu2019,wl_survey}. Therefore, MPNNs cannot distinguish specific graph structures, \emph{e.g.,} $k$-regular graphs. 
Moreover, link prediction (LP) tasks cannot always be answered reliably based on pairs of node embeddings obtained from MPNNs. Specifically, the node automorphism problem arises in instances where two nodes possess identical local structures, resulting in equivalent embeddings and, consequently, identical predictions. However, their relationships to a specific node, \emph{e.g.,} in terms of distance, may differ~\citep{Zhang2021,chamberlain2022}.

Efforts aimed at augmenting the expressiveness of MPNNs have pursued four main directions: aligning with the $k$-WL hierarchy~\citep{morris2019weisfeiler,maron2018invariant, azizian2020expressive}, enriching node features with identifiers, exploiting structural information that cannot be captured by the WL test~\citep{bodnar2021weisfeiler, bodnar2021weisfeilern}, and Subgraph GNNs (SGNNs)~\citep{bevilacqua2021equivariant, guerra2022explainability}. SGNNs are a recent class of expressive GNNs that model graphs through a collection of subgraphs, extracted explicitly or implicitly. 
Subgraph extraction can be achieved, \emph{e.g.,} by removing or marking specific nodes or directly counting specific substructures~\citep{papp-marking}. However, in the worst case, previous SGNNs involve computationally intensive preprocessing steps or running a GNN many times.

\par To address these limitations, we introduce a novel approach grounded in graph signal processing~\citep{ortega2018graph} and spectral graph theory~\citep{chung1997spectral}. Our method introduces a novel eigenbasis, denoted as the \textbf{L}earnable \textbf{L}anczos with \textbf{L}inear \textbf{C}onstraints (LLwLC), which can explicitly encode linear constraints, particularly those derived from extracted induced subgraphs, into the basis. We propose a low-rank approximation~\citep{eckart1936approximation} of the Laplacian matrix based on the Lanczos algorithm with linear constraints~\citep{golub2000large}.

\begin{figure*}[ht]
 \centering
\begin{overpic}[width=0.89\textwidth]{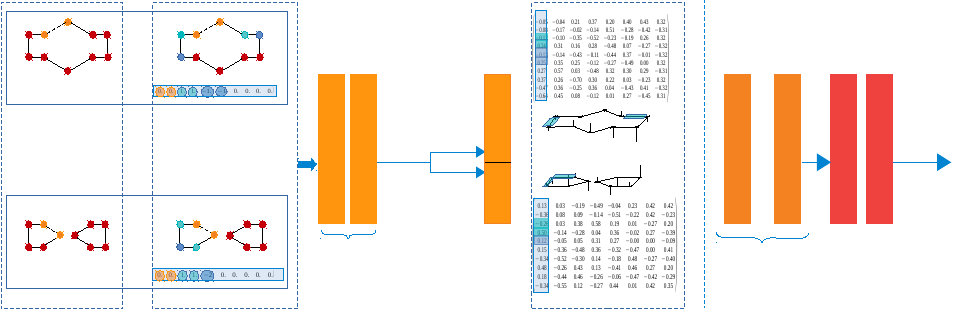}
    \put(38.,21.1){\rotatebox{-90}{ \scriptsize\text{Inner Loop}}}
    \put(36.7,22.5){\rotatebox{-90}{\scriptsize $\mat{P}(\vec{\nu})$ $\rightarrow$ \text{LSQR}}}
    \put(35.,21.1){\rotatebox{-90}{ \scriptsize \text{Outer Loop}}}
    \put(33.6,22.5){\rotatebox{-90}{  \scriptsize \text{Lanczos Steps}}}
  \put(51.5,14.){\rotatebox{-90}{\small $\mat{R}$}}
  \put(51.3,24.3){\rotatebox{-90}{\small $ \mat{V} = \mat{Q}\mat{B}$}}
  \put(81.7,24.9){\rotatebox{-90}{\tiny $\sigma(\mat{V}f_{k}(\mat{R})\mat{V}{\transpose}\mat{X}_{k}\mat{W})$}}
    \put(76.5,24.9){\rotatebox{-90}{\tiny $\sigma(\mat{V}f_1(\mat{R})\mat{V}{\transpose}\mat{X}_1\mat{W})$}}
  \put(87.8,23.5){\rotatebox{-90}{\small Global Pooling}}
  \put(91.5,23.9){\rotatebox{-90}{\small Fully Connected }}
 \put(99.9,16.){\rotatebox{-0}{ $\hat{y}$}}
  \put(40.9,16.5){\scriptsize $\mat{Q}$}
  \put(40.,13.6){\scriptsize $\mat{R}, \mat{B}=$}
  \put(40.0, 12.){\small $\text{evd}(\mat{T})$}
    \put(32., 6.){\small To satisfy $\mat{C}$} 
  \put(18.,1.){\scriptsize One Column of $\mat{C}$}
  \put(2.,1.){\scriptsize Input Graph}
  \put(56.8,0.9){\scriptsize Eigenbasis $\mat{V} = \mat{Q}\mat{B}$}
  \put(74.5, 5.5){\small Stacking of $k$ multiple layers}
  \end{overpic}
  \caption{The comprehensive workflow of LLwLC: We introduce a novel eigenbasis to boost the expressivity of GNNs by introducing constraints. For the two $2$-regular input graphs (first dotted block), the constraint matrix $\mat{C}$ (two columns shown in the second dotted block) is derived. Specifically, as depicted, for the first Neumann eigenvalue constraint (focused on link representation), we integrate the degrees of nodes into one column of the constraint matrix $\bm{C}$ such that $\mat{C}{\transpose}\vec{f}=0$ ensures \( \vec{f}(x) - \vec{f}(y) = 0 \) for adjacent nodes $x$ and $y$ in the boundary. To predict the dotted link, nodes up to two hops away are analyzed: one-hop in light blue and two-hop in dark blue, centered on yellow query nodes. Based on the input graph Laplacian matrix $\mat{L}$ and the constraint matrix $\mat{C}$, we address the eigenvalue problem under these linear constraints using the Lanczos algorithm with the linear constraint consisting of an outer loop (Lanczos algorithm) and an inner loop ($\mat{P}(\nu)$ to $\mathcal{N}(\mat{C})$ with LSQR algorithm) resulting in the eigenpairs $\mat{V} = \mat{Q}\mat{B}$ and $\mat{R}$. The third block demonstrates constructing eigenbases for the two 2-regular graphs, ensuring features conform to constraints (indicated in blue).
  We stack multiple blocks, where block $i$ learns the new features $\mat{X}_{i+1}=\sigma(\mat{V}f_i(\mat{R})\mat{V}{\transpose}\mat{X}_i\mat{W})$ by applying a multi-layer perceptron $f_i$ to the eigenvalue matrix $\mat{R}$. This sequence concludes with global pooling and a fully connected layer to output the predicted link probability $\hat{y}$.}
   \label{loop-graph}
  \end{figure*}
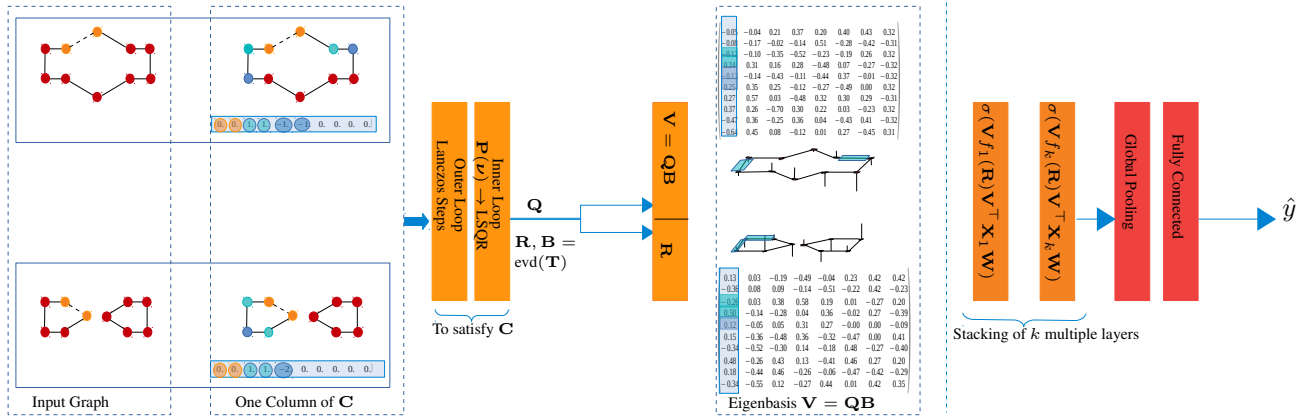

\par The LLwLC eigenbasis enhances feature expressiveness by incorporating linear constraints derived from the graph structure. We propose two novel subgraph extraction policies, focusing on vertex-deleted subgraphs and Neumann eigenvalue constraints. By leveraging vertex-deleted subgraphs, LLwLC can effectively distinguish between graphs that are not separable using the 2-WL test, while Neumann constraints enable the encoding of boundary conditions and link representations between nodes, allowing differentiation of $k$-regular graphs that are indistinguishable by the WL test. Theoretical analysis indicates that LLwLC can be applied in various problem settings. We evaluate its effectiveness in link prediction tasks, where expressivity and particularly the ability to distinguish automorphic nodes is pivotal. 
\par Our research presents several key contributions: 
(i) We develop a novel, efficient eigenbasis (LLwLC) for encoding linear constraints, including induced subgraphs, using the Lanczos algorithm with linear constraints. (ii)  We explore the application of Neumann eigenvalue constraints within this eigenbasis, which encodes induced subgraphs and link representations. (iii) We provide a comprehensive theoretical analysis of LLwLC's convergence. (iv) We investigate the impact of vertex-deleted subgraphs on improving expressiveness of our proposed LLwLC. (v) We conduct extensive experiments showcasing the significant impact of the LLwLC on the challenging link prediction task.

\section{Preliminaries}
\paragraph{Notations} An undirected graph $G(V, E, \mat{X})$ consists of a vertex set $V$, edge set $E$, and node features $\mat{X} \in \mathbb{R}^{n \times d}$, where $n$ is the number of nodes. Each row $\vec{x}_v \in \mathbb{R}^d$ represents the features of node $v \in V$. $\mat{A}$ and $\mat{D}$ are the graph's adjacency and degree matrices, respectively. The graph Laplacian is $\mat{L} = \mat{D} - \mat{A}$, with $\mat{U}$ and $\mat{\Lambda}$ as its eigenvector and eigenvalue matrices. The degree of node $v$ is denoted by $d_v$.
\paragraph{Spectral Graph Convolutional Networks} 
For a graph signal $\vec{x} \in \mathbb{R}^n$,~\citet{shuman2013emerging} define the graph Fourier transform, $\mat{U}\transpose\vec{x}$, and its inverse, $\mat{U}\vec{x}$, based on the eigenbasis of the graph Laplacian matrix $\mat{L}$. 
The graph convolution is $\mat{U} (\mat{U}\transpose\vec{x} \ast \mat{U}\transpose\vec{y}) = \mat{U}g(\mat{\Lambda})\mat{U}\transpose\vec{x}$
where $\vec{y}$ is the graph filter, $g$ is the function applied over the eigenvalue matrix $\mat{\Lambda}$ to encode the graph filter, and $\ast$ is the elementwise multiplication. The seminal spectral GCN method~\cite{bruna2013spectral} is cubic in the number of nodes. To address this, different $g$ functions were defined~\cite{henaff2015deep, azizi20223d}. 
LanczosNet~\cite{liao2019lanczosnet} uses the Lanczos algorithm for fast multi-scale computation with learnable spectral filters.
However, like previous GNNs, LanczosNet has limited expressivity,~\emph{e.g.}, it cannot distinguish between the $k$-regular graphs. To address these limitations and enhance feature expressiveness in GNNs, we introduce a novel learnable spectral basis for encoding subgraphs as linear constraints.
\section{Lanczos with Linear Constraint Networks}
\label{llwlc}
This section outlines the Lanczos Algorithm with Linear Constraints and the construction of a constraint matrix $\mathbf{C}$ for subgraph embedding, followed by the development of the complete LLwLC block and LLwLCNet pipeline. It also includes a proof of convergence properties and explores subgraph extraction policies, focusing on Neumann eigenvalue constraints and vertex-deleted subgraphs.
\begin{algorithm}
\small
\captionof{algorithm}{Lanczos Algorithm with Linear Constraint (LLwLC). Main steps for subgraph encoding in the Graph Laplacian eigenbasis are highlighted in red.}
\label{algo:Lanczos-lin}
\begin{algorithmic}[1]
\STATE \textbf{Input:} $\mat{L}, \mat{P} = \mat{I} - \mat{C}(\mat{C}\transpose\mat{C})^{-1}\mat{C}\transpose$,  random initialized vector $\vec{\nu},$ number of steps $\kappa$, error tolerance $\epsilon$
\STATE \textbf{Init:} $\vec{\nu}_1 = \n{\mat{P}}(\vec{\nu}), \beta_1 = \lVert \vec{\nu}_1 \rVert_2, \vec{q}_0 = 0$
\FOR{$j = 1$ \textbf{to} $\kappa$}
    \STATE $\vec{q}_j = \frac{\vec{\nu}_j}{\beta_j}$
    \STATE \textcolor{red}{$\vec{p}_j = \mat{P}(\mat{L}\vec{q}_j)$} 
    \STATE $\vec{u}_j = \mat{L}\vec{q}_j - \beta_{j}\vec{q}_{j-1}$
    $\rightarrow$ \textcolor{red}{$\vec{p}_j  - \beta_j\vec{q}_{j-1}$}
    \STATE $\alpha_j = \vec{u}_j\transpose\vec{q}_j$
    \STATE $\vec{\nu}_{j+1} = \vec{u}_j - \alpha_j\vec{q}_j$
    \STATE $\beta_{j+1} = \lVert \vec{\nu}_{j+1}\rVert_2$
    \IF{$\beta_{j+1} \leq \epsilon$}
        \STATE \textbf{quit}
    \ENDIF
\ENDFOR
\STATE $Q = [q_1, \dots, q_{\kappa} ]$, Construct $\mat{T}$
\STATE $\text{EVD}(\mat{T}) = \mat{B}\mat{R}\mat{B}\transpose$
\STATE \textbf{Return} $\mat{V} = \mat{Q}\textcolor{red}{\cdot}\mat{B}$ and $\mat{R}$
\end{algorithmic}
\end{algorithm}

\begin{algorithm}
\small
\caption{LLwLCNet}
\label{llwlcnet}
\begin{algorithmic}[1]
\STATE \textbf{Input:} \text{LLwLC output} $\mat{V}, \mat{R}, \text{signal } \mat{X}$
\STATE \textbf{Init:} $\mat{X}_0 = \mat{X}$
\FOR{$i = 0$ \textbf{to} $k - 1$}
    \STATE $\mat{X}_{i+1} = \sigma(\mat{V}f_i(\mat{R})\mat{V}\transpose\mat{X}_i\mat{W}_i)$
\ENDFOR \\
$\hat{y}$ = \text{Fully Connected}(\text{Global Pooling($\mat{X}_{k}$)})
\STATE \textbf{Return} $\hat{y}$
\end{algorithmic}
\end{algorithm}

\subsection{Lanczos Algorithm with Linear Constraints}
For a given symmetric matrix $\mat{L} \in \mathbb{R}^{n\times n}$ and a randomly initialized vector $\vec{\nu} \in \mathbb{R}^n$, the $\kappa$-step Lanczos algorithm~\citep{Lanczos1950} computes an orthogonal matrix $\mat{Q} \in \mathbb{R}^{n \times m}$ and a symmetric tridiagonal matrix $\mat{T} \in \mathbb{R}^{m \times m}$, such that $\mat{Q}\transpose\mat{L}\mat{Q} = \mat{T}$. We represent $\mat{Q}_N = [\vec{q}_1, \dots, \vec{q}_N]$ where the column vector $\vec{q}_i$ corresponds to the $i$\textsuperscript{th} Lanczos vector. The matrices $\mat{B}\in \mathbb{R}^{m \times m}$ and $\mat{R}\in \mathbb{R}^{m \times m}$ represent the eigenvectors and eigenvalues of $\mat{T}$, respectively. By investigating the $j$\textsuperscript{th} column of the system $\mat{L}\mat{Q} = \mat{Q}\mat{T}$ and rearranging terms, we obtain $\mat{L}\vec{q}_j = \beta_{j+1} \vec{q}_{j+1} + \beta_{j}\vec{q}_{j-1} + \alpha_j \vec{q}_j$. 
\par Having the linear constraint changes the plain Lanczos algorithm by replacing $\vec{u}_j = \mat{L}\vec{q}_j - \beta_{j}\vec{q}_{j-1}$  with $\vec{u}_j = \vec{p}_j - \beta_j\vec{q}_{j-1}$ assuming the initial vector $\vec{\nu}$ is projected into the null space of the constraints~\citep{golub2000large}. Please note that the orthogonal projector $\mat{P}$ can be obtained through the QR decomposition of $\mat{C}$ when dealing with a dense constraint matrix $\mat{C}$. In situations where $\mat{C}$ is sparse and $\text{dim}(\mathcal{N}(\mat{C}\transpose)) \approx n$, the projector is given by $\mat{P} = \mat{I} - \mat{C}\mat{C}^{\dagger}$, with $\mat{C}^\dagger$ being the Moore-Penrose inverse of $\mat{C}$. Assuming $\mat{C}$ has full column rank, $\mat{C}^{\dagger}$ can be computed as $(\mat{C}\transpose\mat{C})^{-1}\mat{C}\transpose$~\citep{bjorck1996numerical}. If we project the initial vector $\vec{\nu}$ into null space of the constraint matrix $\vec{\nu}_1 = \mat{P}\vec{\nu} \in \mathcal{N}(\mat{C}\transpose)$ and notice the mathematical equivalence between computing the smallest eigenvalue of the constraint  $A_p = \mat{P}\transpose\mat{L}\mat{P}$ and $\mat{L}$ then one step of the Lanczos algorithm with the linear constraints is $ \beta_{j+1} \vec{q}_{j+1} = \mat{P}\mat{L}\mat{P}\vec{q}_j - \beta_{j}\mat{P}\vec{q}_{j-1} - \alpha_j\mat{P}\vec{q}_j= \mat{P}(\mat{L}\vec{q}_j - \beta_{j}\vec{q}_{j-1} - \alpha_j\vec{q}_j)$. Algorithm~\ref{algo:Lanczos-lin} describes the steps of the Lanczos Algorithm with the Linear Constraints in detail.
\par This algorithm is structured into two main components: the 'outer loop', which is a straightforward Lanczos algorithm iteration, and the 'inner loop', which focuses on resolving the least squares problem expressed as $\mat{P}(\vec{b}) = \underset{y \in \mathbb{R}^l}{\min}{\lVert \mat{C}\vec{y} - \vec{b}\rVert_2}$. Here, $\vec{y}$ is defined as $\mat{C}^\dag\vec{b}$, and $\vec{b}$ is $\mat{L}\vec{q}_j$.

\subsection{LLwLCNet}
\par In this part, we detail our approach to computing the eigenvectors of the graph Laplacian matrix, ensuring they adhere to input graph constraints. We construct our eigenbasis by addressing a large, sparse, symmetric eigenvalue problem with homogeneous linear constraints. It requires minimizing
\begin{equation}
    \underset{\mat{C}{\transpose}\vec{f} = 0, \vec{f} \neq 0}{\min}\frac{\vec{f}{\transpose}\mat{L}\vec{f}}{\vec{f}{\transpose}\vec{f}},
    \label{eig-prob1}
\end{equation}
where $\mat{C} \in \mathbb{R}^{n \times l}$ with $n \gg l$ is the constraint matrix.

To address the problem outlined in Equation~\ref{eig-prob1}, we utilize the Lanczos algorithm with the linear constraints, detailed in Algorithm~\ref{algo:Lanczos-lin}. However, differing from the iterative approach for the least square equation suggested in the Lanczos algorithm with the linear constraints to solve
\begin{equation}
\mat{C}\vec{y} = \vec{b},    
\label{inner-loop}
\end{equation}
we utilize the PyTorch framework~\citep{paszkeautomatic} for our computations. This choice is due to our sparse and not overly large constraint matrix allowing for the direct QR factorization~\citep{anderson1992generalized} within PyTorch, offering numerical stability and the capability for backpropagation. In the following, we describe constructing the constraint matrix $\mathbf{C}$, where we extract subgraphs first and derive the constraint matrix accordingly.

\paragraph{Constraint Matrix $\mathbf{C}$}
The matrix $\mathbf{C}$ allows us to specify linear constraints that the graph Laplacian matrix's eigenvectors must satisfy. Each column of the constraint matrix represents a distinct constraint. We introduce two approaches for defining the constraint matrix $\mathbf{C}$ on the basis of specific subgraphs: 
vertex-deleted subgraphs and Neumann eigenvalue constraints for link representations. A detailed explanation of how each column of the constraint matrix is built for these methods can be found in Section~\ref{subgraph-extraction-policy}.

\par Addressing the eigenvalue problem with linear constraints yields a tridiagonal matrix, denoted as \(\mathbf{T}\), and an orthogonal matrix \(\mathbf{Q}\). The decomposition of matrix \(\mathbf{T}\) produces matrices \(\mathbf{R}\) and \(\mathbf{B}\). Here, \(\mathbf{R}\) represents the Ritz eigenvalues, and \(\mathbf{V} = \mathbf{Q}\mathbf{B}\) forms the eigenbasis that satisfies the constraints imposed by matrix \(\mathbf{C}\). Forcing the eigenvectors to satisfy appropriately defined constraints leads to having different eigenbasis for graphs where the MPNN returns the same features. This is exemplified in the case described in Figure~\ref{loop-graph}, where two $2$-regular graphs yield two different eigenbases.

\paragraph{Full Block} After determining the eigenbasis, we are ready to establish the full block of the Lanczos Layer with Linear Constraint (LLwLC). In this new eigenbasis, we develop spectral filters by applying a multilayer perceptron $f$ to the eigenvalue matrix $\mat{R}$. With these learned filters, we reconstruct our basis and transform the graph signals $\mat{X} \in \mathbb{R}^{m \times n}$ into this basis to extract features that meet our specific constraints. $\hat{\mat{L}}$ is the graph Laplacian matrix computed from the low-rank approximation of the constrained eigenvalue problem. Each LLwLCNet block is 
\begin{equation}
\sigma(\mat{V}f(\mat{R})\mat{V}\transpose \mat{X}\mat{W}) = \sigma(\hat{\mat{L}}\mat{X}\mat{W}).
\end{equation}
Here, $\mat{W} \in \mathbb{R}^{n \times m}$ represents the learnable weight matrix, and $\sigma$ denotes the non-linearity applied in each block (ReLU in our experiments). 
\paragraph{Full Architecture} As represented in Algorithm~\ref{llwlcnet}, we increase the number of blocks to deepen our architecture and capture more complex features. Each block reuses the initially computed eigenbasis and applies a multi-layer perceptron (MLP) to the eigenvalue matrix $\mat{R}$ to reconstruct its corresponding $\hat{\mat{L}}$. Our complete pipeline concludes by a global sort pooling~\citep{zhang2018end} and a fully connected block in the last layer as depicted in Figure~\ref{loop-graph}, used to predict link existence.

\subsection{Lanczos Algorithm with Linear Constraint Convergence}
In this section, we substantiate LLwLC eigenbasis's convergence properties by conducting an error analysis on perturbations and referencing Greenbaum's findings to demonstrate the existence of an exact Lanczos algorithm for any perturbed version. By establishing the upper bound for the Lanczos algorithm's low-rank approximation, we affirm the convergence of our LLwLC eigenbasis.
\paragraph{Perturbation and Error Study} 
The accuracy of the linear least square problem using QR factorization depends on the precision of the QR factorization. As discussed by ~\citet{zhang2020high}, two types of accuracy errors are crucial in QR factorization when solving linear least square problems: The backward error for a matrix \(\mat{Z}\) is defined as \(\frac{\lVert \mat{Z} - \hat{\mat{Q}} \hat{\mat{R}} \rVert} {\lVert \mat{Z} \rVert}\) and the orthogonality error of \(\hat{\mat{Q}}\) is measured by \(\lVert \mat{I} - \hat{\mat{Q}}{\transpose}\hat{\mat{Q}}\rVert\). Ideally, both numerical errors should be zero, but due to roundoff errors and the potential loss of orthogonality in the Gram-Schmidt QR process, the QR factorization might not be sufficiently accurate for solving the linear least square problem.

\par After examining the impact on accuracy, we analyze the theoretical gap between the exact Lanczos algorithm and its perturbed variant due to inexact QR factorization. The inexact QR factorization applied to solve Eq.~\eqref{inner-loop} will impact the accuracy of both the Lanczos vectors and the tridiagonal matrices produced. Consequently, the computed tridiagonal matrix \(\mat{T}_j\) is a perturbed version of the theoretical tridiagonal matrix, denoted as \(\mat{T}^{\ast}_j\), that would be generated by an exact Lanczos iteration.  This relationship can be expressed as \(\mat{T}_j  = \mat{T}^{\ast}_j + \mat{E}_j\), where \(\mat{E}_j\) is the perturbation matrix after the \(j^{th}\) step. The following theorem details the error bounds of the perturbed tridiagonal matrix in comparison to the theoretical exact solution of \(\mat{T}\) after the \(j^{th}\) step of the Lanczos algorithm.
\begin{theorem}
Let $\mathcal{U}$ and $\tilde{\mathcal{U}}$ be the eigenspaces corresponding to the smallest eigenvalues $\lambda$ and $\tilde{\lambda}$ of the symmetric matrices $\mat{L}$ and $\tilde{\mat{L}} = \mat{L} + \mat{E}$, respectively. Then for any $\vec{u} \in \mathcal{U}$ and $\tilde{\vec{u}} \in \tilde{\mathcal{U}}$ with $\lVert \vec{u} \rVert_2 = 1$ and $\lVert \tilde{\vec{u}} \rVert_2 = 1$, we have $    \tilde{\lambda} - \lambda  \approx \bigsum {i}{1}{j} \mat{E}_j(i,i)\vec{u}(i)^2 + 2\bigsum {i}{1}{j-1} \mat{E}_j(i,i+1)\vec{u}(i)\vec{u}(i+1), $
where $\mat{E}_j(s, t)$ is the $(s, t)$ element of $\mat{E}_j$. 
\label{therotical-eigs-diff}
\end{theorem}

After exploring the theoretical gap between exact and perturbed Lanczos algorithms, we investigate Greenbaum's result, which shows that each perturbed Lanczos corresponds to an exact version.
\paragraph{Greenbaum's Results~\citep{greenbaum1989behavior}} The tridiagonal matrix $\mat{T}_j$ generated at the end of the $j$\textsuperscript{th} \textit{finite precision} Lanczos process satisfying $\mat{L}\mat{Q}_j = \mat{Q}_j \mat{T}_j + \beta_{j+1} \vec{q}_{j+1}\vec{e}{\transpose}_j + \mat{F}_j,$ where $\vec{e}_j\transpose$ is a vector with the $j$\textsuperscript{th} component one and all the other components zero, $\mat{F} = (\vec{f}_1, \dots, \vec{f}_j)$ is the perturbation term with $\lVert \vec{f}_j \rVert_2  \leq \epsilon \lVert \mat{L} \rVert_2, \epsilon \ll 1$, is the same as that generated by an exact Lanczos process but with a different matrix $\tilde{\mat{L}}$. The matrices $\mat{L}$ and $\tilde{\mat{L}}$ are close in the sense that for any eigenvalue $\lambda(\tilde{\mat{L}})$ of $\tilde{\mat{L}}$, there is an eigenvalue $\lambda(\mat{L})$ of $\mat{L}$ such that $|\lambda(\tilde{\mat{L}}) - \lambda({\mat{L}})| \leq \lVert \mat{F}_j \rVert_2.$ Therefore, in our case with the constant accuracy of the QR factorization, we can show $\mat{P}\mat{L}\mat{P}\tilde{\mat{Q}}_j = \tilde{\mat{Q}}_j\mat{T}_j + \beta_j\tilde{\vec{q}}_{j+1}\vec{e}_j\transpose + \tilde{\mat{F}}_j,$ where \textbf{$\tilde{\mat{F}}_j = \mathbf{O}(\eta)$} with $\eta$ corresponds to the accuracy of the QR method.

Having established each perturbed Lanczos algorithm corresponds to an exact Lanczos algorithm, we demonstrate the theorem below to bound the approximation error, as discussed in~\citep{liao2019lanczosnet}.
\begin{theorem}
Let $\mat{U}\mat{\Lambda}\mat{U}\transpose$ be the eigendecomposition of an $n \times n$ symmetric matrix $\mat{L}$ with $\mat{\Lambda}_{i,i} = \lambda_i,\lambda_1 \geq \dots \geq \lambda_n$ and $\mat{U} = [\vec{u}_1, \dots , \vec{u}_n ]$. Let $\mat{\mathcal{U}}_j \equiv \text{span} \{\vec{u}_1, \dots , \vec{u}_j\}$. Assume $\kappa$-step Lanczos algorithm starts with vector $\vec{\nu}$ and outputs the orthogonal $\mat{Q} \in \mathbb{R}^{n\times \kappa}$ and tridiagonal matrix $\mat{T} \in \mathbb{R}^{ \kappa\times \kappa}$. For any $j$ with $1 < j < n$ and $\kappa > j$, we have 
\begin{multline*}
\lVert \mat{L} - {\mat{Q}}\mat{T} {\mat{Q}}\transpose \rVert^2_F \\ \leq \bigsum{i}{1}{j} \lambda^2_i \Biggl(\frac{
    \sin (\vec{\nu}, \mathcal{U}_i) \prod^{j-1}_{k=1}\frac{\lambda_k - \lambda_N }{\lambda_k - \lambda_j }}{
\cos(\vec{\nu}, \vec{u}_i)T_{\kappa-i}(1 + 2\gamma_i)
}\Biggr)^2
+ \bigsum{i}{j+1}{N} \lambda^2_i,
\end{multline*}
where $T_{\kappa-i}(x)$ is the Chebyshev Polynomial of degree $\kappa - i$ and $\gamma_i = (\lambda_i - \lambda_{i+1})/(\lambda_{i+1} - \lambda_N )$.
\label{converge-L}
\end{theorem}
\par Based on Greenbaum's results~\citep{greenbaum1989behavior}, for our computed perturbed Lanczos algorithm, exists an exact Lanczos algorithm but for a different matrix. Based on Theorem~\ref{converge-L}, we also cognize the upper bound of the low-rank approximator of the  Lanczos algorithm. Thus, the perturbed Lanczos algorithm, caused by the inaccuracy of the QR method for solving the least square equation, converges to the upper bound of the low-rank approximation of the matrix of the exact Lanczos algorithm.
\subsection{Subgraph Extraction Policy} 
\label{subgraph-extraction-policy}
A subgraph selection policy is a function $\pi\colon \mathcal{G} \rightarrow \mathbb{P}(\mathcal{G})$ assigning to a graph a subset of its subgraphs~\citep{bevilacqua2021equivariant}. Here, $\mathcal{G}$ is the set of all graphs with $n$ nodes or less and $\mathbb{P}(\mathcal{G})$ its power set. Although any linear constraint in the input graph satisfying full rank assumption can be encoded in $\mat{C}$, we propose the following subgraph extraction policies.

\paragraph{Neumann Eigenvalue} The Neumann eigenvalue~\citep{chung1997spectral} is 
\begin{gather*}
\label{neumann-eigenvalue}
   \lambda_{S} =  \inf_{\vec{f} \neq 0} 
 \frac{\bigsumm{x \in S} \vec{f}(x)\mat{L}\vec{f}(x)}{\bigsumm{x \in S}  \vec{f}^{2}(x)d_x}, 
\text{ subject to}\\
\sum_{y \in S, x \in \delta S, y \sim x} (\vec{f}(x) - \vec{f}(y)) = 0 \quad  \text{and} \quad \sum_{x \in S} \vec{f}(x)d_x = 0.
\end{gather*}
The function \(\vec{f}\colon S\cup \delta S \rightarrow \mathbb{R}\) represents the Neumann eigenvector satisfying the Neumann conditions. The vertex boundary, \(\delta S\), of an induced subgraph consists of all vertices not in \(S\) but adjacent to at least one vertex in \(S\). Specifically, the first constraint encodes the link representation. Building on previous link prediction research, we consider nodes that are two hops away from the query nodes, where $S$ represents the one-hop-away nodes, and $\delta S$ denotes the boundary nodes between one-hop and two-hop-away nodes. It ensures that the aggregate of eigenvector differences across nodes equates to zero, as expressed by $\sum (\vec{f}(x) - \vec{f}(y)) = 0$. The column corresponding to the subgraph is constructed based on the equation. The degrees of the nodes involved in the constraint are entered into the column to satisfy the equation. Specifically, the degrees of nodes two-hop-away are negated, and the degrees of nodes one-hop-away are included only if they are connected to nodes two hops away. Entries for uninvolved nodes are set to $0$.

\colorlet{darkgreen}{green!60!black}
\begin{example}
 Consider the graph \vcentered{\begin{tikzpicture}  
  [scale=.1,auto=center,every node/.style={circle, fill=red, minimum size=1mm, inner sep=1pt}] 
    
  \node [blue](a1) at (0,5)   {};  
  \node [orange](a2) at (1.5,5) {}; 
  \node [orange](a3) at (3,6) {};  
  \node [blue](a4) at (4.5,5) {};  
  \node [green!60!black](a5) at (6,5)   {};  
  
  \node [darkgreen](a6) at (0,3){};  
  \node (a7) at (1.5,3) {}; 
  \node (a8) at (3.,2)   {};  
  \node (a9) at (4.5,3) {};  
  \node (a10) at (6,3)  {};  

  \draw[red] (a1) -- (a2);  
  \draw (a7) -- (a8);
  \draw (a4) -- (a5);  
  \draw[] (a1) -- (a6);
  \draw[] (a6) -- (a7);
  \draw[] (a8) -- (a9);
  \draw[red] (a2) -- (a3);
  \draw[red] (a3) -- (a4);
  \draw (a9) -- (a10);
  \draw (a5) -- (a10);    
\end{tikzpicture}}, where the query nodes are \textcolor{orange}{orange}, the nodes in $S$ are \textcolor{blue}{blue}, and the boundary nodes $\delta S$ are \textcolor{darkgreen}{green}. The boundary constraint $\sum_{\textcolor{blue}{y} \in S, \textcolor{darkgreen}{x} \in \delta S, \textcolor{blue}{y} \sim \textcolor{darkgreen}{x}} (\textcolor{blue}{\vec{f}(x)} - \textcolor{darkgreen}{\vec{f}(y)}) = 0 $ is realized by the column $[\textcolor{orange}{0},\textcolor{orange}{0},\textcolor{blue}{1},\textcolor{blue}{1},\textcolor{darkgreen}{-1},\textcolor{darkgreen}{-1},\textcolor{red}{0},\textcolor{red}{0},\textcolor{red}{0},\textcolor{red}{0}]\transpose$ of the matrix $\mathbf{C}$.
\end{example}

 For the induced subgraph case, we populate its corresponding column in $\mathbf{C}$ with the degrees of each node (and fill the remaining entries of the column with $0$ for nodes not involved in the subgraph), ensuring that the derived eigenvectors fulfill the condition $\sum \vec{f}(x)d_x = 0$ (detailed in Vertex-deleted Subgraphs). This condition ensures that the features learned using this eigenbasis reflect the imposed constraints.

Thus, the Neumann eigenvalue problem can be reformulated as~\autoref{eig-prob1}
where \(\mat{L} \in \mathbb{R}^{n\times n}\) is a symmetric and large sparse matrix, and \(\mat{C} \in \mathbb{R}^{n \times l}\) (with \(n \gg l\)) is also large, sparse, and of full column rank. 
The time complexity involved in extracting subgraphs depends on the product of the maximum degree of nodes and the count of nodes in the boundary.
When we enforce that the eigenvectors satisfy the constraints related to induced subgraphs and link representation, we ensure that the corresponding features adhere to these constraints.
\begin{proposition}
Applying Neumann eigenvalue constraints to the eigenbasis results in features that exhibit greater expressivity than MPNNs. Besides addressing the node automorphism problem~\citep{Srinivasan2020On}, these enhanced features enable the distinction of specific $k$-regular graphs from each other, thereby significantly enhancing expressivity in GNNs (Proof in Appendix). 
\label{neumann-features}
\end{proposition}

\paragraph{Vertex-deleted Subgraphs}
\label{universal-approximator}

 The second subgraph extraction policy we propose is based on vertex-deleted subgraphs.  Given that we can encode any subgraph into our eigenbasis, we can examine whether a specific substructure collection can completely characterize each graph. By the reconstruction conjecture~\citep{ulam1960collection}, we assume we can reconstruct the graph if we have all the $n - 1$ vertex-deleted subgraphs. According to~\citet{Illya2023exponentially}, exponentially many graphs are determined by their spectrum. Thus, we propose that each column of the constraint matrix corresponds to a vertex-deleted subgraph, and we include the degrees of the nodes within the corresponding subgraph in each column, representing deleted nodes by zero entries.
 \begin{example}
 Consider the graph \vcentered{%
 \begin{tikzpicture}  
  [scale=.1,auto=center,every node/.style={circle, fill=red, minimum size=1mm, inner sep=1pt}] 
    
  \node [orange](a1) at (0,5)   {};  
  \node [orange](a2) at (1.5,5) {}; 
  \node [orange](a3) at (3,6) {};  
  \node [orange](a4) at (4.5,5) {};  
  \node (a5) at (6,5)   {};  
  
  \node (a6) at (0,3){};  
  \node (a7) at (1.5,3) {}; 
  \node (a8) at (3.,2)   {};  
  \node (a9) at (4.5,3) {};  
  \node (a10) at (6,3)  {};  

  \draw[red] (a1) -- (a2);  
  \draw (a7) -- (a8);
  \draw (a4) -- (a5);  
  \draw[] (a1) -- (a6);
  \draw[] (a6) -- (a7);
  \draw[] (a8) -- (a9);
  \draw[red] (a2) -- (a3);
  \draw[red] (a3) -- (a4);
  \draw (a9) -- (a10);
  \draw (a5) -- (a10);    
\end{tikzpicture}}, where the \textcolor{red}{red} nodes are deleted, and the \textcolor{orange}{orange} nodes are contained in the subgraph. The constraint
$\textcolor{orange}{\sum\vec{f}(x)d_x = 0}$ is represented by the column $[\textcolor{orange}{2},\textcolor{orange}{2},\textcolor{orange}{2},\textcolor{orange}{2},\textcolor{red}{0},\textcolor{red}{0},\textcolor{red}{0},\textcolor{red}{0},\textcolor{red}{0},\textcolor{red}{0}]\transpose$ in the matrix $\mat{C}$.
\end{example}
We build the other columns of the constraint matrix accordingly. The intuition behind defining such a constraint is that it implicitly represents the vertex deleted subgraph. Learning the filters over the eigenvalue matrix of the graph Laplacian leads to defining an affinity matrix based on Gaussian functions and driving the corresponding adjacency matrix based on the threshold~\cite{liao2019lanczosnet}. Thus, $\mat{L} = \mat{I} - \mat{D}^{-1} \mat{W}$ where $\mat{W}$ is the affinity matrix~\cite{coifman2006diffusion}. $ \mat{P} = \mat{D}^{-1} \mat{W}$ is the Markov transition matrix and its powers $\mat{P}^t$ are the probability of $t$-hop-away distance between the nodes $i$ to node $j$.  Thus, the diffusion map is defined based on the right eigenvector and eigenvalues of the Markov transition matrix $\mat{P}$. If $\lambda_i$ and $\vec{\psi}_i$ are the eigenpairs of the matrix $\mat{P}$ then the eigenbasis of diffusion map after $t$-steps for the $i$\textsuperscript{th} value is $\mat{\Phi}_t(i) = (\lambda_1^t\psi_1(t), \lambda_2^t\psi_2(t), ...)$. Projecting to null space of the constraint matrix, $(\mat{I} - (\mat{C}\transpose(\mat{C}\mat{C}\transpose)^{-1}\mat{C}))$ cancels out the nodes with zero entries and considers the others based on the degree we put in the constraint matrix. The diffusion map embedding adapts to reflect diffusion patterns that are aware of the vertex deletions.

\par In extracting vertex-deleted subgraphs from a dense graph, the worst-case time complexity is $\mathcal{O}(n^3)$ due to the high number of edges in dense graphs, leading to increased computational load. For sparse graphs with fewer edges, the complexity is reduced by a factor of $\mathcal{O}(n)$. Our empirical results show that applying a few constraints significantly improves performance, making the practical overhead align linearly with the number of nodes. This approach is similar to the marker-based method~\citep{papp-marking}, as both consider slightly perturbed input graphs. 

Utilizing vertex-deleted subgraphs allows us to distinguish between graphs that are not distinguishable with the 2-WL method. For instance, although the \(4 \times 4\) rook and Shirkhande graphs are indistinguishable with 2-WL, their corresponding constraint matrices—constructed from the vertex-deleted subgraphs—have different spectra and thus different constraint matrices. These differences in spectra lead to distinct eigenbases, enabling us to differentiate between these graphs with our novel LLwLC.

\paragraph{Stochastic Constraints \mat{C}} In our study, we empirically discovered that effective outcomes can be attained with a limited number of constraints, such as Neumann constraints, or in another ablation study involving only ten vertex-deleted subgraphs. This observation aligns with the findings of~\citet{bollobas1990almost}, who demonstrated that almost all graphs can be reconstructed using only three vertex-deleted subgraphs. Given the prohibitive cost of considering all $n$ vertex-deleted subgraphs, our approach, inspired by~\citep{bevilacqua2021equivariant}, involves selecting a subset of constraints for consideration. At each epoch, we stochastically choose $k$ vertex-deleted subgraphs for analysis.

\begin{table}[]
    \caption{Time Complexity Comparison}
    \centering
    \begin{subtable}[t]{0.45\textwidth}
    \caption{LP Tasks' Time Complexity: Number of Nodes (n), Edges (E), Dimensionality of Node Features (d), Propagation Hops (l), Sketch Size in BUDDY (h)~\cite{chamberlain2022}, Eigenvectors ($\kappa$), and Constraints ($k$).}   
        \centering
    \scalebox{0.5}{
    \begin{tabular}{lcccccr}
    \toprule
    Time Complexity  & SEAL  & BUDDY & NBFNet & LLwLC \\ 
    \midrule
    Preprocessing           &  $\mathcal{O}(1)$  & $\mathcal{O}(lE(d+h))$ & $\mathcal{O}(1)$ &$ \mathcal{O}(1)$\\
    Training (1 link)       &  $\mathcal{O}(Ed^2)$& $\mathcal{O}(l^2h + ld^2)$  & $\mathcal{O}(Ed+nd^2)$ &$ \mathcal{O}(\kappa E + k^2n)$\\
    Inference               & $\mathcal{O}(Ed^2)$ & $\mathcal{O}(l^2h+ld^2)$ & $\mathcal{O}(Ed+nd^2)$ & $ \mathcal{O}(\kappa E + k^2n)$\\
    \bottomrule
    \end{tabular}}
    \label{complexitytable}
    \end{subtable}
 \begin{subtable}[t]{0.49\textwidth}
 \caption{Expressivity Models' Time Complexity.}
        \centering
    \scalebox{0.65}{
     \begin{tabular}{lccccccr}
    \toprule
     Time Complexity   &$k$-IGN  & $k$-GNN  & LLwLC  & GSN (3-WL)\\
    \midrule
       Preprocessing & -  & - &$ \mathcal{O}(n^3)$& $\mathcal{O}(n^k)$ & \\  
       Dense & $\mathcal{O}(n^k)$  & $\mathcal{O}(n^k)$ &$ \mathcal{O}(n^3)$& $\mathcal{O}(n^2)$ & \\
       Sparse & $\mathcal{O}(n^k)$  & $\mathcal{O}(n^k)$ &$ \mathcal{O}(n)$& $\mathcal{O}(n)$  &\\
    \bottomrule
    \end{tabular}}
    \label{complexitytable-express}
\end{subtable}
\end{table}
\paragraph{Time Complexity} The time complexity of our method is $\mathcal{O}(\kappa E + k^2n + k^3)$ for the outer loop (Lanczos algorithm), the QR factorization, and computing the pseudo-inverse of $k \times k$ matrix, respectively, where $\kappa \ll n$ is the number of computed eigenvectors and $k \ll n$ is the number of linear constraints. Table~\ref{complexitytable} compares the time complexities of LLwLCNet and other link prediction methods.

\par Aligned with~\citep{barabasi1999emergence} research and supported by our empirical experiments, it has been established that graph reconstruction can be effectively achieved with just a few constraints. This leads to time complexity linear to the number of nodes. Contrastingly, established $k$-WL expressive models, such as $k$-IGNs~\citep{maron2019provably} and $k$-GNNs~\citep{morris2019weisfeiler}, are known to have a higher time complexity of $\mathcal{O}(n^k)$. 
\par Furthermore, we compare LLwLCNet's time complexity with GSN~\citep{bouritsas2022improving}, which improves expressivity by counting specific substructures. However, it relies on task-specific substructure selection to introduce a suitable inductive bias and deals with the subgraph isomorphism. In GSN, the preprocessing step, in general, is $\mathcal{O}(n^k)$ for a generic substructure of size $k$. In contrast, our experimental results demonstrate that LLwLCNet can capture graph properties without requiring task-specific prior knowledge. Comparisons are in Table~\ref{complexitytable-express}.

\section{Experiments}
\par To demonstrate the effectiveness of our method in addressing node automorphism and leveraging subgraphs, we conducted experiments focused on the link prediction task. We compared the performance of LLwLCNet against traditional heuristics (CN~\citep{barabasi1999emergence}, RA~\citep{zhou2009predicting}, AA~\citep{adamic2003friends}), vanilla GNNs (GCN~\citep{kipf2016semi}, SAGE~\citep{NIPS2017_5dd9db5e}), GNNs modifying the input graph of MPNNs (SEAL~\citep{zhang2018link}, NBFNet~\citep{zhu2021neural}), and GNNs with manual features as pairwise representations (Neo-GNN~\citep{yun2021neo}, BUDDY~\citep{chamberlain2022}). Baseline results are from~\citet{chamberlain2022}. Our evaluation includes five link prediction benchmarks: Cora, Citeseer, Pubmed~\citep{yang2016revisiting} (Planetoid datasets), OGBL-Collab, and OGBL-Vessel~\citep{NEURIPS2020_fb60d411}. Dataset statistics are shown in the appendix. Baseline results for OGBL-Collab and OGBL-Vessel are from the OGB leaderboard. 
\paragraph{Setup} 
In our experiments, we used a learning rate of 0.001 for 20 training epochs. The model has two 32-channel MLP layers, each with ReLU non-linearities and dropout. We cap the eigenpairs at 10, padding missing pairs with zeros. Both PyTorch~\citep{paszkeautomatic} and PyTorch Geometric~\citep{Fey/Lenssen/2019} were used in our implementation. The model was trained using the binary cross entropy loss. Following SEAL, we altered 10\% of links for test data and used the remaining 90\% for training.
\paragraph{Link Prediction Results} 
While the LLwLC framework theoretically applies to various problem domains, our experiments focus on the LP tasks to demonstrate its effectiveness in addressing node automorphism, considering substructures, and thus improving expressiveness. 
Consistent with prior studies~\citep{ chamberlain2022}, considering two-hop nodes is sufficient for effective LP task.

As depicted in Table~\ref{result-table}, LLwLC stands out as a robust framework for link prediction, consistently delivering high performance on link prediction benchmarks with metrics given in the first row. 
With only 0.02M parameters, LLwLC outperforms previous models on the Planetoid dataset, highlighting the significance of encoding subgraph structures and selective subsets of node relations for superior link prediction results. Our model also achieves state-of-the-art performance on the OGBL-Collab dataset with just 0.02M, compared to BUDDY (1.10M) and SEAL (0.50M). Increasing blocks to 0.03M yields a 67.50\% HR@50 score. On OGBL-Vessel, we achieve competitive results with just 0.019M and 10\% of the training data, demonstrating LLwLC's effectiveness in a lightweight architecture.

\begin{table*}
\caption{Results on LP benchmarks; LLwLCNet w. Neumann Constraints. 
The colors denote the \textcolor{red}{best} and \textcolor{blue}{second-best} models. LLwLCNet trained with $10\%$ of the VESSEL dataset.
}
\begin{center}
\scalebox{0.65}{
\begin{small}
\begin{sc}
\begin{tabular}{lcccccccr}
\toprule
           & Cora   & Citeseer & Pubmed & Collab  & Vessel \\
Method     & HR@100 & HR@100   & HR@100 & HR@50 & roc-auc  \\
\midrule
CN         & 33.92  & 29.79    & 23.13  & 56.44 &   48.49  \\
AA         & 39.85  & 35.19    & 27.38  & 64.35  &  48.49 \\
RA         & 41.07  & 33.56    & 27.03  & 64.00  & n.a.   \\
\midrule
GCN~\citep{kipf2016semi}        & 66.79  & 67.08    & 53.02  & 44.75 & 43.53  \\
SAGE~\citep{hamilton2017inductive}       & 55.02  & 57.01    & 39.66  & 48.10  & 49.89  \\
\midrule 
Neo-GNN~\citep{yun2021neo}    & 80.42  & 84.67    & 73.93  &  57.52  & n.a.\\
SEAL~\citep{zhang2018link}       & 81.71  & 83.89    & \textcolor{blue}{75.54}&64.74& 80.50\\
NBFnet~\citep{zhu2021neural}     & 71.65  & 74.07    & 58.73  &  OOM &  n.a. \\
Surel+~\citep{yin2023surel+}      & N.A. &  N.A.  & N.A.  &   64.10 &   \n{85.73}     \\
BUDDY~\citep{chamberlain2022}      & \textcolor{blue}{88.00}  & \textcolor{blue}{92.93}    & 74.10  &  \textcolor{blue}{65.94} & 55.14 \\
\midrule
LLwLCNet  & \textcolor{red}{91.44}  & \textcolor{red}{93.40} & \textcolor{red}{83.10}& \textcolor{red}{66.86}& \textcolor{blue}{81.60}      \\
\midrule
\# Params.     & $0.019$M  & $0.018$M    & $0.024$M  & $0.026$M & $0.036$M\\
\bottomrule
\end{tabular}
\end{sc}
\end{small}}
\end{center}

\label{result-table}
\end{table*}

\paragraph{Ablation Studies} To demonstrate the impact of the applied constraints, we compare the eigenbasis estimation using the ground truth Laplacian matrix $\mat{L}$ (with edges provided) against the approximation of the eigenbasis using  $\mat{L}$ alongside Neumann constraints $\mat{C}$. This comparison is conducted on three benchmark datasets, as detailed in Table~\ref{table-1}. We notice that including Neumann constraints, $\mat{C}$, markedly enhances the results. 
\par In our extended research, we explore the impact of adding additional constraints to our framework. As detailed in Table~\ref{table-constraints}, we explore the influence of integrating an increased number of constraints using vertex-deleted subgraphs. Our observations reveal that implementing merely ten constraints from vertex-deleted subgraphs yields state-of-the-art enhancements in benchmark datasets. Notably, within the OGBL-Collab dataset, the constraints significantly boost performance metrics (\emph{e.g.,} HR@50 increases from $42.83$ in LanczosNet, which does not apply linear constraints to the input graph, to $69.40$ with the incorporation of ten vertex-deleted subgraph constraints). Similar improvements are evident in the Cora dataset, with HR@100 rising from $90.80$ to $93.10$, and PubMed shows similar improved results over the baseline. 
Our results are consistent with the theoretical framework. It significantly aligns with~\citep{bollobas1990almost}, suggesting nearly all graphs can be reconstructed using three vertex-deleted subgraphs.

\begin{table}[]
    \caption{Subgraph Constraint Impact. \textcolor{red}{Best} in red.}
    \centering
    \begin{subfigure}{0.52\textwidth}
    \begin{subtable}[t]{\linewidth}
    \caption{Results with/without Neumann constraints (with $\mathbf{L}$).}
        \centering
\scalebox{0.6}{
\begin{tabular}{lcccr}
\toprule
           & Cora                & Citeseer          & PubMed      \\
Method      & AUC                 & AUC               & AUC         \\
\midrule
LanczosNet  &         94.5\%          &   96.5\%   &   97.2\%\\
LLwLC (w. $\mat{L}$ \& $\mat{C}$)   &   \textcolor{red}{97.0\%}   &    \textcolor{red}{98.1\%}  &  \textcolor{red}{98.3\%} \\
\bottomrule
\end{tabular}}
\label{table-1}    
\end{subtable}
\end{subfigure}
 \begin{subtable}[t]{0.52\textwidth}
        \centering
        \caption{Results with different subgraph constraints.}
\scalebox{0.5}{
\begin{tabular}{lcccr}
\toprule
 & Collab   &  Cora    & PubMed   \\
Method                       & HR@50         &  HR@100  & HR@100  \\
\midrule
LanczosNet~\citep{liao2019lanczosnet}    &   42.58 &90.80 &77.18  \\
LLwLCNet w. Neumann Constraints          &   66.86 & 91.44 &\textcolor{red}{83.10}\\
LLwLCNet w. 10 Constraints               &   \textcolor{red}{69.40} & \textcolor{red}{93.10}&82.28\\
\midrule
\# Params.                   & 0.026M        &   0.019M&   0.021M  \\
\bottomrule
\end{tabular}}
    
\label{table-constraints}
\end{subtable}
\end{table}

\begin{table}[h]
\caption{Wall-time comparisons with LP models. Training time is measured for one epoch. The LLwLC model demonstrates state-of-the-art performance in both benchmark datasets while reducing the training data (Reduction) requirement by up to $90\%$, using only $10\%$ of the data.}
\centering
\scalebox{0.54}{
\begin{tabular}{lccccccr}
\toprule
\textbf{Dataset}&   & SEAL                          & BUDDY   &  GCN    & LLwLC  & LLwLC  & LLwLC   \\
\midrule
&\text{Pre-train(s)}    &    0             &  5      &    0    &     0   & 0& 0\\
&\text{Train (s)} &    81                           &  1      &  66      &   7    & 4 &  1\\
\textbf{PubMed}& Reduction   &    100    &  100    &  100  &  10  & 5 & 1\\
 & Accuracy   & 75.54   &  74.10 & 53.02 & \n{80.15}    & 78.20 & 70.84\\
& \# Params.  & 0.486M   & 1.565M &  0.052M  &0.024M & 0.024M   & 0.024M\\
\midrule  
&\text{Train (s)} &    12600&  N.A. & 12600 & 1200   &700 &140 \\
\textbf{OGBL-Vessel}  & Reduction &   100   &  100 & 100&  10  & 5 & 1\\
& Accuracy  &    80.50  & 55.14& 43.53   & \n{81.60}&   79.24 & 78.72\\
 &\# Params.   & 0.042M  & N.A. &  0.035M  & 0.036M &  0.036M  & 0.036M\\
\bottomrule
\end{tabular}}
\label{walltime}
\end{table}
\paragraph{Wall-time and Reduction Dataset}
Our study has shown an enhancement in model expressivity, leading to a decrease in the need for parameters. This results in less training data being required. As shown in Table~\ref{walltime}, on the PubMed dataset, our model surpasses BUDDY~\citep{chamberlain2022} and SEAL~\citep{zhang2018link} using merel $10\%$ of the training data, which results in a training speed that is 20 times faster than SEAL. Furthermore, it delivers performance on par with BUDDY while using just $1\%$ of the training data within the same training duration and without requiring any preprocessing steps. This enhanced efficiency reflects a more efficient learning mechanism and significantly reduces training duration. Regarding the OGBL-Vessel dataset, our approach outperforms SEAL~\citep{zhang2018link}, demonstrating a 10-fold increase in training speed. It achieves similar outcomes at a speed 90 times faster using only 1\% of the training data. These experiments were all carried out on a GeForce GT1030 GPU (CUDA 11.6, PyTorch 1.13).

\section{Related Work}

\paragraph{MPNN Expressivity} The expressivity of GNNs is typically expressed in terms of their ability to distinguish non-isomorphic graphs. 
As no polynomial-time algorithm for solving the graph isomorphism problem is known, developing GNNs that are both expressive and efficient poses a major challenge.
\citet{xu2018how} found that the expressivity of MPNNs is limited to that of the 1-WL test. This limitation is crucial in real-world applications, as the 1-WL test cannot distinguish certain structures, such as regular graphs, and does not capture several natural graph properties well, such as distances and cycle counts~\cite{li2022expressive}. Recent studies have addressed these limitations with four approaches: adding random attributes to nodes~\cite{sato2021random}, using deterministic positional features~\cite{zhang2018link}, developing higher-order GNNs to surpass the 1-WL test's expressivity limits~\cite{maron2018invariant}, and subgraph GNNs applying markings, such as the node-deletion approach of ESAN~\cite{bevilacqua2021equivariant}. Please refer to the recent survey by \citet{wl_survey} for a comprehensive overview of these techniques. Our approach aligns with the direction of subgraph GNNs.
\paragraph{Subgraph GNNs for Link Prediction} 
 SEAL enhances WLNM~\cite{zhang2017weisfeiler}, the first subgraph-based LP, by using graph convolutional layers and encoding positional features. SEAL demonstrates that information within two-hop subgraphs is sufficient, aligning with classical methods. Neo-GNN~\cite{yun2021neo} and BUDDY~\cite{chamberlain2022} decouple pairwise representation from node representation learning to reduce computational overhead but may oversimplify pairwise representations.
\section{Conclusion}
In this work, we introduced the Learnable Lanczos algorithm with Linear Constraints (LLwLC), a novel method designed to enhance the expressivity of Graph Neural Networks (GNNs). Through the incorporation of two novel subgraph extraction strategies, we  managed to construct a lightweight architecture that minimizes reliance on extensive training datasets. Empirical results show that our method significantly improves performance in link prediction tasks across various benchmark datasets. Notably, the LLwLC achieved 20$\times$ and 10$\times$ speedup, requiring only 5\% and 10\% data from the PubMed and OGBL-Vessel datasets respectively, compared to the state-of-the-art methods. These findings underscore the practical utility and theoretical advancement of our method, illustrating the LLwLC's potential as a more expressive approach than 2-WL and its ability to differentiate between $k$-regular graphs. The advancements made with the LLwLC not only represent a significant contribution to the field but also set a promising direction for future exploration and development in the realm of GNNs. As a further future work, we believe investigating the impact of learning linear constraints between nodes and edges within the input graph and encoding them to the eigenbasis of the graph Laplacian matrix, can promise further advancements in this field.

\textbf{Acknowledgment} This research was funded by the Austrian Research Promotion Agency (FFG) under project no. 874065. The author sincerely appreciates the valuable feedback from Mohsen Fayyaz, Joshua Erde, and Andr{\'a}s Papp.

\bibliography{example_paper}

\begin{thebibliography}{58}
\providecommand{\natexlab}[1]{#1}
\providecommand{\url}[1]{\texttt{#1}}
\expandafter\ifx\csname urlstyle\endcsname\relax
  \providecommand{\doi}[1]{doi: #1}\else
  \providecommand{\doi}{doi: \begingroup \urlstyle{rm}\Url}\fi

\bibitem[Adamic \& Adar(2003)Adamic and Adar]{adamic2003friends}
Adamic, L.~A. and Adar, E.
\newblock {Friends and Neighbors on the Web}.
\newblock \emph{Social networks}, 2003.

\bibitem[Anderson et~al.(1992)Anderson, Bai, and Dongarra]{anderson1992generalized}
Anderson, E., Bai, Z., and Dongarra, J.
\newblock {Generalized QR Factorization and its Applications}.
\newblock \emph{Linear Algebra and its Applications}, 1992.

\bibitem[Azizi et~al.(2022)Azizi, Possegger, Rodol{\`a}, and Bischof]{azizi20223d}
Azizi, N., Possegger, H., Rodol{\`a}, E., and Bischof, H.
\newblock {3D Human Pose Estimation using M{\"o}bius Graph Convolutional Networks}.
\newblock In \emph{ECCV}, 2022.

\bibitem[Azizian \& Lelarge(2021)Azizian and Lelarge]{azizian2020expressive}
Azizian, W. and Lelarge, M.
\newblock {Expressive Power of Invariant and Equivariant Graph Neural Networks}.
\newblock In \emph{ICLR}, 2021.

\bibitem[Barab{\'a}si \& Albert(1999)Barab{\'a}si and Albert]{barabasi1999emergence}
Barab{\'a}si, A.-L. and Albert, R.
\newblock {Emergence of Scaling in Random Networks}.
\newblock \emph{science}, 1999.

\bibitem[Bevilacqua et~al.(2022)Bevilacqua, Frasca, Lim, Srinivasan, Cai, Balamurugan, Bronstein, and Maron]{bevilacqua2021equivariant}
Bevilacqua, B., Frasca, F., Lim, D., Srinivasan, B., Cai, C., Balamurugan, G., Bronstein, M.~M., and Maron, H.
\newblock {Equivariant Subgraph Aggregation Networks}.
\newblock In \emph{ICLR}, 2022.

\bibitem[Bj{\"o}rck(1996)]{bjorck1996numerical}
Bj{\"o}rck, {\AA}.
\newblock \emph{{Numerical Methods for Least Squares Problems}}.
\newblock SIAM, 1996.

\bibitem[Bodnar et~al.(2021{\natexlab{a}})Bodnar, Frasca, Otter, Wang, Lio, Montufar, and Bronstein]{bodnar2021weisfeilern}
Bodnar, C., Frasca, F., Otter, N., Wang, Y., Lio, P., Montufar, G.~F., and Bronstein, M.
\newblock {Weisfeiler and Lehman Go Cellular: CW Networks}.
\newblock In \emph{NeurIPS}, 2021{\natexlab{a}}.

\bibitem[Bodnar et~al.(2021{\natexlab{b}})Bodnar, Frasca, Wang, Otter, Montufar, Lio, and Bronstein]{bodnar2021weisfeiler}
Bodnar, C., Frasca, F., Wang, Y., Otter, N., Montufar, G.~F., Lio, P., and Bronstein, M.
\newblock {Weisfeiler and Lehman Go Topological: Message Passing Simplicial Networks}.
\newblock In \emph{ICML}, 2021{\natexlab{b}}.

\bibitem[Bollob{\'a}s(1990)]{bollobas1990almost}
Bollob{\'a}s, B.
\newblock {Almost Every Graph Has Reconstruction Number Three}.
\newblock \emph{Journal of Graph Theory}, 1990.

\bibitem[Bouritsas et~al.(2022)Bouritsas, Frasca, Zafeiriou, and Bronstein]{bouritsas2022improving}
Bouritsas, G., Frasca, F., Zafeiriou, S.~P., and Bronstein, M.
\newblock {Improving Graph Neural Network Expressivity via Subgraph Isomorphism Counting}.
\newblock \emph{TPAMI}, 2022.

\bibitem[Bruna et~al.(2013)Bruna, Zaremba, Szlam, and LeCun]{bruna2013spectral}
Bruna, J., Zaremba, W., Szlam, A., and LeCun, Y.
\newblock {Spectral Networks and Locally Connected Networks on Graphs}.
\newblock In \emph{ICLR}, 2013.

\bibitem[Chamberlain et~al.(2023)Chamberlain, Shirobokov, Rossi, Frasca, Markovich, Hammerla, Bronstein, and Hansmire]{chamberlain2022}
Chamberlain, B.~P., Shirobokov, S., Rossi, E., Frasca, F., Markovich, T., Hammerla, N.~Y., Bronstein, M.~M., and Hansmire, M.
\newblock {Graph Neural Networks for Link Prediction with Subgraph Sketching}.
\newblock In \emph{ICLR}, 2023.

\bibitem[Chung \& Graham(1997)Chung and Graham]{chung1997spectral}
Chung, F.~R. and Graham, F.~C.
\newblock \emph{{Spectral Graph Theory}}.
\newblock American Mathematical Soc., 1997.

\bibitem[Coifman \& Lafon(2006)Coifman and Lafon]{coifman2006diffusion}
Coifman, R.~R. and Lafon, S.
\newblock {Diffusion Maps}.
\newblock \emph{Applied and computational harmonic analysis}, 2006.

\bibitem[Eckart \& Young(1936)Eckart and Young]{eckart1936approximation}
Eckart, C. and Young, G.
\newblock {The Approximation of One Matrix by Another of Lower Rank}.
\newblock \emph{Psychometrika}, 1936.

\bibitem[Fey \& Lenssen(2019)Fey and Lenssen]{Fey/Lenssen/2019}
Fey, M. and Lenssen, J.~E.
\newblock {Fast Graph Representation Learning with {PyTorch Geometric}}.
\newblock In \emph{ICLRW}, 2019.

\bibitem[Golub et~al.(2000)Golub, Zhang, and Zha]{golub2000large}
Golub, G.~H., Zhang, Z., and Zha, H.
\newblock {Large Sparse Symmetric Eigenvalue Problems with Homogeneous Linear Constraints: the Lanczos Process with Inner-Outer Iterations}.
\newblock \emph{Linear Algebra and its Applications}, 2000.

\bibitem[Greenbaum(1989)]{greenbaum1989behavior}
Greenbaum, A.
\newblock {Behavior of Slightly Perturbed Lanczos and Conjugate-Gradient Recurrences}.
\newblock \emph{Linear Algebra and its Applications}, 1989.

\bibitem[Guerra et~al.(2022)Guerra, Spinelli, Scardapane, and Bianchi]{guerra2022explainability}
Guerra, M., Spinelli, I., Scardapane, S., and Bianchi, F.~M.
\newblock {Explainability in Subgraphs-enhanced Graph Neural Networks}.
\newblock \emph{arXiv:2209.07926}, 2022.

\bibitem[Hamilton et~al.(2017{\natexlab{a}})Hamilton, Ying, and Leskovec]{NIPS2017_5dd9db5e}
Hamilton, W., Ying, Z., and Leskovec, J.
\newblock {Inductive Representation Learning on Large Graphs}.
\newblock In \emph{NeurIPS}, 2017{\natexlab{a}}.

\bibitem[Hamilton et~al.(2017{\natexlab{b}})Hamilton, Ying, and Leskovec]{hamilton2017inductive}
Hamilton, W.~L., Ying, R., and Leskovec, J.
\newblock {Inductive Representation Learning on Large Graphs}.
\newblock In \emph{NeurIPS}, 2017{\natexlab{b}}.

\bibitem[Henaff et~al.(2015)Henaff, Bruna, and LeCun]{henaff2015deep}
Henaff, M., Bruna, J., and LeCun, Y.
\newblock {Deep Convolutional Networks on Graph-structured Data}.
\newblock \emph{arXiv:1506.05163}, 2015.

\bibitem[Hu et~al.(2020)Hu, Fey, Zitnik, Dong, Ren, Liu, Catasta, and Leskovec]{NEURIPS2020_fb60d411}
Hu, W., Fey, M., Zitnik, M., Dong, Y., Ren, H., Liu, B., Catasta, M., and Leskovec, J.
\newblock {Open Graph Benchmark: Datasets for Machine Learning on Graphs}.
\newblock In \emph{NeurIPS}, 2020.

\bibitem[Kipf \& Welling(2017)Kipf and Welling]{kipf2016semi}
Kipf, T.~N. and Welling, M.
\newblock {Semi-Supervised Classification with Graph Convolutional Networks}.
\newblock In \emph{ICLR}, 2017.

\bibitem[Koval \& Kwan(2023)Koval and Kwan]{Illya2023exponentially}
Koval, I. and Kwan, M.
\newblock {Exponentially Many Graphs Are Determined by Their Spectrum}.
\newblock \emph{arXiv:2309.09788}, 2023.

\bibitem[Lanczos(1950)]{Lanczos1950}
Lanczos, C.
\newblock {An Iteration Method for the Solution of the Eigenvalue Problem of Linear Differential and Integral Operators}.
\newblock \emph{J.Res.Natl.Bur.Stand.}, 1950.

\bibitem[Li \& Leskovec(2022)Li and Leskovec]{li2022expressive}
Li, P. and Leskovec, J.
\newblock {The Expressive Power of Graph Neural Networks}.
\newblock \emph{Graph Neural Networks: Foundations, Frontiers, and Applications}, 2022.

\bibitem[Liao et~al.(2019)Liao, Zhao, Urtasun, and Zemel]{liao2019lanczosnet}
Liao, R., Zhao, Z., Urtasun, R., and Zemel, R.~S.
\newblock {Lanczosnet: Multi-scale Deep Graph Convolutional Networks}.
\newblock In \emph{ICLR}, 2019.

\bibitem[Maron et~al.(2019{\natexlab{a}})Maron, Ben-Hamu, Serviansky, and Lipman]{maron2019provably}
Maron, H., Ben-Hamu, H., Serviansky, H., and Lipman, Y.
\newblock {Provably Powerful Graph Networks}.
\newblock In \emph{NeurIPS}, 2019{\natexlab{a}}.

\bibitem[Maron et~al.(2019{\natexlab{b}})Maron, Ben-Hamu, Shamir, and Lipman]{maron2018invariant}
Maron, H., Ben-Hamu, H., Shamir, N., and Lipman, Y.
\newblock {Invariant and Equivariant Graph Networks}.
\newblock In \emph{ICLR}, 2019{\natexlab{b}}.

\bibitem[Monti et~al.(2017)Monti, Bronstein, and Bresson]{monti2017geometric}
Monti, F., Bronstein, M., and Bresson, X.
\newblock {Geometric Matrix Completion with Recurrent Multi-Graph Neural Networks}.
\newblock In \emph{NeurIPS}, 2017.

\bibitem[Morris et~al.(2019)Morris, Ritzert, Fey, Hamilton, Lenssen, Rattan, and Grohe]{morris2019weisfeiler}
Morris, C., Ritzert, M., Fey, M., Hamilton, W.~L., Lenssen, J.~E., Rattan, G., and Grohe, M.
\newblock {Weisfeiler and Leman Go Neural: Higher-order Graph Neural Networks}.
\newblock In \emph{AAAI}, 2019.

\bibitem[Morris et~al.(2023)Morris, Lipman, Maron, Rieck, Kriege, Grohe, Fey, and Borgwardt]{wl_survey}
Morris, C., Lipman, Y., Maron, H., Rieck, B., Kriege, N.~M., Grohe, M., Fey, M., and Borgwardt, K.
\newblock Weisfeiler and leman go machine learning: The story so far.
\newblock \emph{Journal of Machine Learning Research}, 24\penalty0 (333):\penalty0 1--59, 2023.
\newblock URL \url{http://jmlr.org/papers/v24/22-0240.html}.

\bibitem[Nickel et~al.(2015)Nickel, Murphy, Tresp, and Gabrilovich]{nickel2015review}
Nickel, M., Murphy, K., Tresp, V., and Gabrilovich, E.
\newblock {A Review of Relational Machine Learning for Knowledge Graphs}.
\newblock \emph{IEEE}, 2015.

\bibitem[Ortega et~al.(2018)Ortega, Frossard, Kova{\v{c}}evi{\'c}, Moura, and Vandergheynst]{ortega2018graph}
Ortega, A., Frossard, P., Kova{\v{c}}evi{\'c}, J., Moura, J.~M., and Vandergheynst, P.
\newblock {Graph Signal Processing: Overview, Challenges, and Applications}.
\newblock \emph{IEEE}, 2018.

\bibitem[Oyetunde et~al.(2017)Oyetunde, Zhang, Chen, Tang, and Lo]{oyetunde2017boostgapfill}
Oyetunde, T., Zhang, M., Chen, Y., Tang, Y., and Lo, C.
\newblock {BoostGAPFILL: Improving the Fidelity of Metabolic Network Reconstructions Through Integrated Constraint and Pattern-based Methods}.
\newblock \emph{Bioinformatics}, 2017.

\bibitem[Papp \& Wattenhofer(2022)Papp and Wattenhofer]{papp-marking}
Papp, P.~A. and Wattenhofer, R.
\newblock {A Theoretical Comparison of Graph Neural Network Extensions}.
\newblock In \emph{ICML}, 2022.

\bibitem[Parlett(1980)]{parlett1980symmetric}
Parlett, B.~N.
\newblock \emph{The Symmetric Eigenvalue Problem}.
\newblock Prentice-Hall, 1980.

\bibitem[Paszke et~al.(2017)Paszke, Gross, Chintala, Chanan, Yang, DeVito, Lin, Desmaison, Antiga, and Lerer]{paszkeautomatic}
Paszke, A., Gross, S., Chintala, S., Chanan, G., Yang, E., DeVito, Z., Lin, Z., Desmaison, A., Antiga, L., and Lerer, A.
\newblock {Automatic Differentiation in PyTorch}.
\newblock In \emph{NeurIPS}, 2017.

\bibitem[Sato et~al.(2021)Sato, Yamada, and Kashima]{sato2021random}
Sato, R., Yamada, M., and Kashima, H.
\newblock {Random Features Strengthen Graph Neural Networks}.
\newblock In \emph{SDM}, 2021.

\bibitem[Shibata et~al.(2012)Shibata, Kajikawa, and Sakata]{shibata2012link}
Shibata, N., Kajikawa, Y., and Sakata, I.
\newblock {Link Prediction in Citation Networks}.
\newblock \emph{JASIST}, 2012.

\bibitem[Shuman et~al.(2013)Shuman, Narang, Frossard, Ortega, and Vandergheynst]{shuman2013emerging}
Shuman, D.~I., Narang, S.~K., Frossard, P., Ortega, A., and Vandergheynst, P.
\newblock {The Emerging Field of Signal Processing on Graphs: Extending High-dimensional Data Analysis to Networks and Other Irregular Domains}.
\newblock \emph{IEEE}, 2013.

\bibitem[Srinivasan \& Ribeiro(2020)Srinivasan and Ribeiro]{Srinivasan2020On}
Srinivasan, B. and Ribeiro, B.
\newblock {On the Equivalence between Positional Node Embeddings and Structural Graph Representations}.
\newblock In \emph{ICLR}, 2020.

\bibitem[Ulam(1960)]{ulam1960collection}
Ulam, S.~M.
\newblock \emph{{A Collection of Mathematical Problems}}.
\newblock Interscience Publishers, 1960.

\bibitem[Weisfeiler \& Leman(1968)Weisfeiler and Leman]{weisfeiler1968reduction}
Weisfeiler, B. and Leman, A.
\newblock {The Reduction of a Graph to Canonical Form and the Algebra which Appears Therein}.
\newblock \emph{nti, Series}, 1968.

\bibitem[Xu et~al.(2019{\natexlab{a}})Xu, Hu, Leskovec, and Jegelka]{Xu2019}
Xu, K., Hu, W., Leskovec, J., and Jegelka, S.
\newblock How powerful are graph neural networks?
\newblock In \emph{ICLR}, 2019{\natexlab{a}}.

\bibitem[Xu et~al.(2019{\natexlab{b}})Xu, Hu, Leskovec, and Jegelka]{xu2018how}
Xu, K., Hu, W., Leskovec, J., and Jegelka, S.
\newblock {How Powerful are Graph Neural Networks?}
\newblock In \emph{ICLR}, 2019{\natexlab{b}}.

\bibitem[Yang et~al.(2016)Yang, Cohen, and Salakhudinov]{yang2016revisiting}
Yang, Z., Cohen, W., and Salakhudinov, R.
\newblock {Revisiting Semi-supervised Learning with Graph Embeddings}.
\newblock In \emph{ICML}, 2016.

\bibitem[Yin et~al.(2023)Yin, Zhang, Wang, and Li]{yin2023surel+}
Yin, H., Zhang, M., Wang, J., and Li, P.
\newblock {SUREL+: Moving from Walks to Sets for Scalable Subgraph-based Graph Representation Learning}.
\newblock In \emph{VLDB}, 2023.

\bibitem[Yun et~al.(2021)Yun, Kim, Lee, Kang, and Kim]{yun2021neo}
Yun, S., Kim, S., Lee, J., Kang, J., and Kim, H.~J.
\newblock {Neo-gnns: Neighborhood Overlap-aware Graph Neural Networks for Link Prediction}.
\newblock In \emph{NeurIPS}, 2021.

\bibitem[Zhang \& Chen(2017)Zhang and Chen]{zhang2017weisfeiler}
Zhang, M. and Chen, Y.
\newblock {Weisfeiler-Lehman Neural Machine for Link Prediction}.
\newblock In \emph{ACM SIGKDD}, 2017.

\bibitem[Zhang \& Chen(2018)Zhang and Chen]{zhang2018link}
Zhang, M. and Chen, Y.
\newblock {Link Prediction based on Graph Neural Networks}.
\newblock In \emph{NeurIPS}, 2018.

\bibitem[Zhang et~al.(2018)Zhang, Cui, Neumann, and Chen]{zhang2018end}
Zhang, M., Cui, Z., Neumann, M., and Chen, Y.
\newblock {An End-to-End Deep Learning Architecture for Graph Classification}.
\newblock In \emph{AAAI}, 2018.

\bibitem[Zhang et~al.(2021)Zhang, Li, Xia, Wang, and Jin]{Zhang2021}
Zhang, M., Li, P., Xia, Y., Wang, K., and Jin, L.
\newblock Labeling trick: {A} theory of using graph neural networks for multi-node representation learning.
\newblock In \emph{NeurIPS}, 2021.

\bibitem[Zhang et~al.(2020)Zhang, Baharlouei, and Wu]{zhang2020high}
Zhang, S., Baharlouei, E., and Wu, P.
\newblock {High Accuracy Matrix Computations on Neural Engines: A Study of QR Factorization and Its Applications}.
\newblock In \emph{HPDC}, 2020.

\bibitem[Zhou et~al.(2009)Zhou, L{\"u}, and Zhang]{zhou2009predicting}
Zhou, T., L{\"u}, L., and Zhang, Y.-C.
\newblock {Predicting Missing Links via Local Information}.
\newblock \emph{The European Physical Journal B}, 2009.

\bibitem[Zhu et~al.(2021)Zhu, Zhang, Xhonneux, and Tang]{zhu2021neural}
Zhu, Z., Zhang, Z., Xhonneux, L.-P., and Tang, J.
\newblock {Neural Bellman-ford Networks: A General Graph Neural Network Framework for Link Prediction}.
\newblock In \emph{NeurIPS}, 2021.

\end{thebibliography}
\bibliographystyle{icml2025}

\newpage
\appendix
\onecolumn

\label{appendix}
\section{Theoretical Analyses}
\textbf{Proof of Theorem~\ref{therotical-eigs-diff}.} 
It is the immediate result of the following theorem discussed in~\cite{golub2000large}.
\begin{theorem}
Let $\mathcal{U}$ and $\tilde{\mathcal{U}}$ be the eigenspaces corresponding to the smallest eigenvalues $\lambda$ and $\tilde{\lambda}$ of the symmetric matrices $\mat{A}$ and $ \tilde{\mat{A}} = \mat{A} + \mat{E}$, respectively. Then

1. For any $\vec{u} \in \mathcal{U}$ and $\tilde{\vec{u}} \in \tilde{\mathcal{U}}$ with $\lVert \vec{u} \rVert_2 = \lVert \tilde{\vec{u}} \rVert_2 = 1$,
\begin{equation*}
\tilde{\vec{u}}\transpose \mat{E}\tilde{\vec{u}} \leq \tilde{\lambda} - \lambda \leq \vec{u}\transpose\mat{E}\vec{u}.
\end{equation*}

2. For any $\tilde{\vec{u}} \in \tilde{\mathcal{U}}$ with $ \lVert \vec{u} \rVert_2 = 1$, there exists $\vec{u} \in \mathcal{U}$ with $\lVert \tilde{\vec{u}} \rVert_2 = 1$ such that
\begin{equation*}
\beta \leq \lVert \vec{u} - \tilde{\vec{u}} \rVert_2 \leq \beta (1 + \frac{\beta^2}{1 + \sqrt{2}}),
\end{equation*}
where $\beta$ satisfies

$\text{max} \{ 0,
\frac{\lVert \mat{E}\tilde{\vec{u}} \rVert_2 - |\tilde{\lambda} - \lambda|}{\tilde{d}_{\text{max}}}$, $\frac{\lVert \mat{E}\vec{u} \rVert_2 - |\tilde{\lambda} - \lambda|}{d_{\text{max}}}\}\leq \beta \leq \text{min} \{ \frac{\lVert \mat{E}\hat{\vec{u}} \rVert_2}{\tilde{d}_{\text{min}}},\frac{\lVert \mat{E}\vec{u} \rVert_2}{d_{\text{min}}}\}$

with

$\tilde{d}_{\text{min}} = \text{min}\{|\tilde{\lambda} - \lambda(\mat{A})| | \lambda(\mat{A}) \neq \lambda\}$,

$\tilde{d}_{\text{max}} = \text{max}\{|\tilde{\lambda} - \lambda(\mat{A})| | \lambda(\mat{A}) \neq \lambda\}$,

$d_{\text{min}} = \text{min}\{|\lambda - \lambda(\tilde{\mat{A}}) | | \lambda(\tilde{\mat{A}}) \neq \tilde{\lambda}\}$,

$d_{\text{max}} = \text{max}\{|\lambda - \lambda(\tilde{\mat{A}}) | | \lambda(\tilde{\mat{A}})  \neq \tilde{\lambda}\}$.
\end{theorem}

\textbf{Proof of Theorem~\ref{converge-L}}
As addressed by~\cite{liao2019lanczosnet, parlett1980symmetric}, we have $\mat{L}\mat{Q} = \mat{Q}\mat{T}$ from the Lanczos algorithm. Therefore,
\begin{equation*}
	\lVert \mat{L} - \mat{Q}\mat{T} \mat{Q}^{\transpose}\rVert^{2}_F = \lVert \mat{L} - \mat{L}\mat{Q}\mat{Q}^{\transpose}\rVert^{2}_F = \lVert \mat{L}(\mat{I} - \mat{Q}\mat{Q}^{\transpose})\rVert^{2}_F 
\end{equation*}

Let $\mat{P}^{\perp}_{\mat{Q}} \equiv \mat{I} - \mat{Q}\mat{Q}^{\transpose}$, the orthogonal projection onto the orthogonal complement of subspace $\text{span}\{\mat{Q}\}$. Relying on the eigendecomposition we have,

\begin{flalign*}
\lVert \mat{L} - \mat{Q}\mat{T}\mat{Q}^{\transpose} \rVert^2_F = \lVert \mat{U}\mat{\Lambda}\mat{U}^{\transpose}(\mat{I} - \mat{Q}\mat{Q}^{\transpose})\rVert^2_F = \\
\lVert \Lambda \mat{U}^{\transpose}(\mat{I} - \mat{Q}\mat{Q}^{\transpose})\rVert^{2}_F=
\lVert(\mat{I} - \mat{Q}\mat{Q}^{\transpose})\mat{U}\Lambda\rVert^{2}_F= \\
  \lVert[\lambda_1\mat{P}^{\perp}_{\mat{Q}} \vec{u}_1, \dots , \lambda_N \mat{P}^{\perp} _{\mat{Q}} \vec{u}_N]\rVert^{2}_F,
\end{flalign*}
where we use the fact that $\lVert R\mat{A} \rVert^{2}_F = \lVert \mat{A} \rVert^{2}_F$
for any orthogonal matrix $\mat{R}$ and $\lVert \mat{A}\rVert^{2}_F = \lVert \mat{A} \rVert^{2}_F.$ 
Note that for any $j$ we have,
\begin{align*}
\lVert [\lambda_1\mat{P}^{\perp}_{\mat{Q}}\vec{u}_1, \dots , \lambda_N \mat{P}^{\perp}_{\mat{Q}}\vec{u}_N] \rVert^2_F = \\\bigsum{i}{1}{N}\lambda_i^2\lVert \mat{P}^{\perp}_{\mat{Q}} \vec{u}_i \rVert^2\leq \lambda_i^2\lVert \mat{P}^{\perp}_{\mat{Q}} \vec{u}_i\rVert^2 +  \bigsum{i}{j+1}{N} \lambda_i^2,
\end{align*}
where we use the fact that for any $i, \lVert \mat{P}^{\perp}_{\mat{Q}} \vec{u}_i\rVert^2 = \lVert \vec{u}_i\rVert^2 -\lVert \vec{u}_i - \mat{P}^{\perp}_{\mat{Q}} \vec{u}_i\rVert^2 \\\leq \lVert \vec{u}_i\rVert^2 = 1.$
Note that we have $\text{span}\{\mat{Q}\}=\text{span} \{\nu, \mat{L}\nu, \dots, \mat{L}_{K-1}\nu\} \equiv \kappa_{K}$ from the Lanczos algorithm. Therefore, we have, 
\begin{equation*}
\lVert \mat{P}^{\perp}_{\mat{Q}} \vec{u}_i \rVert = |\text{sin} (\vec{u}_i , \kappa_K)| \leq \\|\text{tan} (\vec{u}_i , \kappa_K)|.
\end{equation*}
We finish the proof by applying the above lemma with $\mat{A} = \mat{L}$.

\section{Proof of Proposition~\ref{neumann-features}.}
\label{prrof-neumann}

We prove the two statements of Proposition~\ref{neumann-features} separately.

\begin{proposition}
    Let $\mathcal{M}_{\text{LLwLC}}$ be the family of LLwLC models and $\mathcal{M}_{\text{MPNN}}$
that of MPNNs. $\mathcal{M}_{\text{LLwLC}}$ is more powerful than $\mathcal{M}_{\text{MPNN}}$
($\mathcal{M}_{\text{MPNN}}\sqsubseteq \mathcal{M}_{\text{LLwLC}}$).
\end{proposition}

\begin{proof}
To show that LLwLCs are strictly more expressive than the 1-WL test, it suffices to show that (i) LLwLC can distinguish a pair of graphs that 1-WL deems isomorphic, and (ii) LLwLC can distinguish all graphs that 1-WL can distinguish.

\par MPNNs cannot distinguish the two graphs in Figure~\ref{loop-graph}, as all nodes have degree two, resulting in uniform node embeddings. Conversely, LLwLC, utilizing the Neumann eigenvalue constraints, effectively discriminates between these graphs. We consider nodes that are two hops away from the query nodes, where $S$ represents the one-hop-away nodes, and $\delta S$ denotes the boundary nodes between one-hop and two-hop-away nodes. Specifically, the first Neumann eigenvalue constraint, $\bigsumm{}{}{} \vec{f}(x) - \vec{f}(y) = 0$, for the bottom graph is $\bigsumm{}{}{} \vec{f}(x) - \vec{f}(y) =  (\vec{f}(4) - \vec{f}(3)) + (\vec{f}(4) - \vec{f}(2)) = 2 \vec{f}(4) - \vec{f}(3) - \vec{f}(2) = 0$, which can be written in the vector form as $\vec{c}{\transpose}\vec{f} = 0$ (Please note that $\vec{f}(4)$ corresponds to dark blue node  two hops away while $\vec{f}(1)$ and $\vec{f}(2)$ corresponds to light blue node one hop away). A similar computation for the second graph yields a different vector representation, $\bigsumm{}{}{} \vec{f}(x) - \vec{f}(y) =  (\vec{f}(2) - \vec{f}(4)) + (\vec{f}(3) - \vec{f}(5)) = \vec{f}(2) + \vec{f}(3) - \vec{f}(4) - \vec{f}(5) = 0$, demonstrating LLwLC's discriminative capability (Please note that $\vec{f}(4)$ and $\vec{f}(5)$ correspond to dark blue nodes two hops away while $\vec{f}(1)$ and $\vec{f}(2)$ correspond to light blue nodes one hop away).
The same constraints are generated for all edges of the two input graphs, allowing distinguishable graph representations to be obtained.

\par Similarly, in LLwLC, after reconstructing the eigenbasis by projecting the eigenbasis of a graph's Laplacian matrix into the null space of the constraint matrix defined by Neumann eigenvalue constraints, each node considers its neighbor nodes to update its features. This method positions LLwLC to be at least as expressive as MPNNs, meaning it can solve any sample solvable by MPNNs (The constrained eigenvalue problem is equivalent to quadratic programming with linear constraints, where projecting to the null space preserves graph connectivity and the graph Laplacian. The linearly independent and complementary constraint columns maintain local connectivity, making the model at least as expressive as the 1-WL test. Notably, this projection allows each node to consider both neighboring node features and those in the constraint matrix). 

\end{proof}

\par Furthermore, as proved by~\cite{Srinivasan2020On}, the model family of MPNNs suffers from the automorphic node problem. Let $\mathcal{M}$ be a family of models. We say $\mathcal{M}$ suffers from the automorphic node problem if for any model $M \in \mathcal{M}$, and any simple (attributed) graph $G = (V, E, \mathbf{A})$ we have that $\forall (u_1, v_1), (u_2, v_2) \in V \times V, \{ u_1 \sim_G u_2 \land v_1 \sim_G v_2\} \Rightarrow M((u_1, v_1)) = M((u_2, v_2))$, where $v_1\sim_G v_2$ denotes that there is an automorphism $\varphi$ of $G$ with $\varphi(v_1)=v_2$~\cite{chamberlain2022}.

The LLwLC method effectively tackles the node automorphism problem. 
This method's effectiveness stems from leveraging Neumann eigenvalue constraints, distinguishing edge representations among nodes separated by $k$-hops-away and those $(k+1)\text{-hops-away}$ ($k=1$ in this example), lying at the graph's boundary. This distinction in edge representation across different orbits within the $C_6$ graph, as indicated in Figure~\ref{orbits_BUDDY}, is crucial for the LLwLC approach to address node automorphism effectively. 

\begin{proposition}
 The family of LLwLC models does not suffer from the automorphic node problem.
\end{proposition}
\begin{proof}
It suffices to provide an example of a LLwLC model and a graph $G$ such that the model distinguishes two node pairs, whose nodes are automorphic. Consider the example in Figure~\ref{orbits_BUDDY}, where all nodes are in the same single orbit w.r.t.~the automorphism group $\Aut(G)$.
However, when $\Aut(G)$ acts on node pairs, a non-trivial orbit partition is induced, and the ability to distinguish node pairs in different orbits is crucial for LP tasks.
For the pair (2, 6) in the red orbit, its first Neumann eigenvalue constraint is significantly related to the pairs (1, 3) and (3, 5). Meanwhile, the pair (2, 4) in the black orbit has a first Neumann eigenvalue constraint that associates the pairs (1, 3) and (5, 6). On the other hand, the pair (2, 5) in the green orbit primarily focuses on the pair (1, 3) regarding its first Neumann eigenvalue constraint. Thus, different representations can be learned for node pairs in different orbits.
\end{proof}
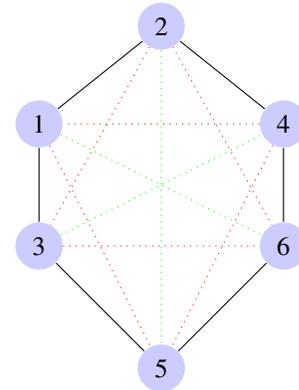
\begin{figure}[ht]
  \centering
\begin{tikzpicture}  
  [scale=.26,auto=center, every node/.style={circle,fill=blue!20}]   
    
  \node (a2) at (1,5)  {1};  
  \node (a3) at (3.5,7)  {2};  
  \node (a4) at (1,2.5) {3};  
  \node (a5) at (6,5)  {4};  
  \node (a6) at (3.5,0)  {5};  
  \node (a7) at (6,2.5)  {6};  
  
  \draw (a2) -- (a3);  
  \draw (a2) -- (a4);  
  \draw (a4) -- (a6);  
  \draw (a3) -- (a5);  
  \draw (a6) -- (a7);  
  \draw (a5) -- (a7);

  \draw[draw = red, dotted ] (a3) -- (a7);
  \draw[draw = red, dotted ] (a3) -- (a4);
  \draw[draw = red, dotted ] (a4) -- (a7);
  \draw[draw = red, dotted ] (a5) -- (a2);
  \draw[draw = red, dotted ] (a5) -- (a6);
  \draw[draw = red, dotted ] (a2) -- (a6);

  \draw[draw = green, dotted ] (a3) -- (a6);
  \draw[draw = green, dotted ] (a2) -- (a7);
  \draw[draw = green, dotted ] (a4) -- (a5);
\end{tikzpicture}  
\caption{The $C_6$ graph shows three different orbits of node pairs with three different colors. Both BUDDY~\cite{chamberlain2022} features and our proposed features can distinguish them.}
\label{orbits_BUDDY}
\end{figure}
\section{Dataset Details.} 
\label{datasetss}
\begin{table*}[]
\caption{ Properties of LP benchmarks. Confidence intervals are +/- one standard deviation. Splits for the Planetoid datasets are random and Collab uses the fixed OGB splits.}
\begin{center}
\scalebox{0.36}{
\resizebox{\textwidth}{!}{
\begin{sc}
\begin{tabular}{lccccccr}
\toprule
           & Cora     & Citeseer & Pubmed    & Collab       \\
\midrule
\# Nodes   & 2708     & 3327     & 18717     & 235868          \\
\# Edges   & 5278     & 4676     & 44,327    & 1,285,465     \\
splits     & rand     & rand     & rand      & time         \\
avg deg    &  3.9     & 2.74     & 4.5       & 5.45        \\
avg deg    &  15.21   & 7.51     & 20.25     & 29.70  \\
1-hop size &  12+/-15 & 8+/-8    & 12+/-17   & 99 +/-251  \\
2-hop size &127+/-131 & 58+/-92  & 260+/-432 & 115+/-571  \\
\bottomrule
\end{tabular}
\end{sc}}}
\end{center}
\label{dataset-benchmarks}
\end{table*}
\par The fundamental characteristics of the experimental datasets, as well as the details regarding the increase in subgraph dimensions with respect to the number of hops, are presented in Table~\ref{dataset-benchmarks}.

\section{Lanczos Algorithm}
\par For a given symmetric matrix $\mat{L} \in \mathbb{R}^{n\times n}$ and a randomly initialized vector $\vec{v} \in \mathbb{R}^n$, the n-step Lanczos algorithm computes an orthogonal matrix $\mat{Q} \in \mathbb{R}^{n \times m}$ and a symmetric tridiagonal matrix $\mat{T} \in \mathbb{R}^{m \times m}$, such that $\mat{Q}^{\transpose}\mat{L}\mat{Q} = \mat{T}$.
We denote $\mat{Q} = [\vec{q}_1, \dots, \vec{q}_N]$ where column vector $\vec{q}_i$ is the $i^{\text{th}}$ Lanczos vector. $\mat{T}$ is the tridiagonal matrix with the eigenvector and eigenvalue matrices $\mat{B}\in \mathbb{R}^{m \times m}$ and $\mat{R}\in \mathbb{R}^{m \times m}$, respectively. $\mat{Q}$ forms an orthonormal basis of the Krylov subspace $\mathcal{K}_n(\mat{L}, \vec{b})$ and its first K columns form the orthonormal basis of $\mathcal{K}(\mat{L}, \vec{x})$. By investigating the $j^{\text{th}}$ column of the system $\mat{L}\mat{Q} = \mat{Q}\mat{T}$ and rearranging terms, we obtain $\mat{L}\vec{q}_j = \beta_{j+1} \vec{q}_{j+1} + \beta_{j}\vec{q}_{j-1} + \alpha_j \vec{q}_j$, and the first $j$ steps of the Lanczos process take the form $\mat{L}\mat{Q}_j = \mat{Q}_j \mat{T}_j + \beta_{j+1} \vec{q}_{j+1}\vec{e}^{\transpose}_j$~\cite{liao2019lanczosnet}. 
One can verify that $\mat{Q}$ forms an orthonormal basis of the Krylov subspace $\mathcal{K}_N(\mat{L}, \vec{b})$ and the first K columns of $\mat{Q}$ forms the orthonormal basis of $\mathcal{K}(\mat{L}, \vec{x})$. Intuitively, if we investigate the $j^{\text{th}}$ column of the system $\mat{L}\mat{Q} = \mat{Q}\mat{T}$ and rearrange terms, we obtain $\mat{L}\vec{q}_j = \beta_{j+1} \vec{q}_{j+1} + \beta_{j}\vec{q}_{j-1} + \alpha_j \vec{q}_j$, and the first $j$ steps of the Lanczos process take the form $\mat{L}\mat{Q}_j = \mat{Q}_j \mat{T}_j + \beta_{j+1} \vec{q}_{j+1}\vec{e}^{\transpose}_j$~\cite{liao2019lanczosnet}.

\begin{algorithm}[]
\captionof{algorithm}{Lanczos Algorithm}
\label{algo:Lanczos-}
\begin{algorithmic}[1]
\STATE \textbf{Input}{$\mat{L}, \vec{\nu}, \kappa, \epsilon$}
\STATE \textbf{Init}{$\vec{\nu}_1 = \vec{\nu}, \beta_1 = \lVert \vec{\nu}_1 \rVert_2, \vec{q}_0 = 0$.}
  \FOR{ $j = 1$ to $\kappa$}
    \STATE $\vec{q}_j = \frac{\vec{\nu}_j}{\beta_j}$
    \STATE $\vec{u}_j = \mat{L}\vec{q}_j - \beta_{j}\vec{q}_{j-1}$ 
    \STATE $\alpha_j = \vec{u}_j\transpose\vec{q}_j$ 
    \STATE $\vec{\nu}_{j+1} = \vec{u}_j - \alpha_j\vec{q}_j$ 
    \STATE $\beta_{j+1} = \lVert \vec{\nu}_{j+1}\rVert_2$  
    \STATE If $\beta_{j+1} \leq \epsilon$, quit
  \ENDFOR
  \STATE $\mat{Q} = [\vec{q}_1, \dots, \vec{q}_{\kappa} ]$\\
  \STATE Construct $\mat{T}$\\
  \STATE $\mat{B}\mat{R}\mat{B}\transpose = \text{evd}(\mat{T})$\\
  \STATE \textbf{Return} $\mat{V} = \mat{Q}\cdot\mat{B}$ and $\mat{R}$
\end{algorithmic}
\end{algorithm}

\end{document}